%% file: delay_arxiv.tex
\documentclass[11pt]{article}
\usepackage[margin=1.56in]{geometry}
\usepackage[hypertexnames=false]{hyperref}
\usepackage{color}
\usepackage{latexsym}
\usepackage{dsfont}
\usepackage{amssymb}
\usepackage{float}
\usepackage{algorithm,algorithmic}
\usepackage{graphicx}
\usepackage{amsmath, amsfonts,amssymb,theorem,euscript,array,enumerate,amsfonts,mathrsfs}
\usepackage[dvipsnames]{xcolor}
\usepackage{subcaption}

\usepackage[T1]{fontenc}
\usepackage[latin1]{inputenc}
\usepackage{hhline}
\usepackage{graphicx}
\usepackage{verbatim}
\usepackage{eurosym}
\usepackage{dsfont}

\usepackage{algorithm}
\usepackage{algorithmic}
\def\I{\mathbb{I}}
\def\N{\mathbb{N}}
\def\V{\mathbb{V}}
\def\P{\mathbb{P}}
\def\E{\mathbb{E}} 
\def\R{\mathbb{R}} 
\def\B{\mathbb{B}} 
 
\def \HH{\mathcal{H}}
\def \FF{\mathcal{F}}
\def\EE{\mathcal{E}} 
\def\II{\mathcal{I}} 
\def\NN{\mathcal{N}} 
\usepackage{mathtools}
\usepackage{amsfonts}       

\usepackage{hyperref}
\usepackage{amsmath}
\usepackage{booktabs}

\newtheorem{theorem}{Theorem}

\newtheorem{proposition}{Proposition}
\newtheorem{assumption}{Assumption}

\newtheorem{corollary}{Corollary}
\newtheorem{remark}{Remark}

\newenvironment{proof}[1][{\it Proof.}]{\begin{trivlist}
\item[\hskip \labelsep {\bfseries #1}]}{ \hfill
$\Box$\end{trivlist}\vskip -0.2 cm}

\begin{document}
\title{Delay-Adaptive Learning in \\
	Generalized Linear Contextual Bandits} 
\author{Jose Blanchet
\thanks{Department of Management Science and Engineering, Stanford University, USA. \textbf{Email:} jose.blanchet@stanford.edu}
\and
Renyuan Xu
\thanks{Mathematical Institute, University of Oxford, UK. \textbf{Email:} xur@maths.ox.ac.uk}
\and
Zhengyuan Zhou
\thanks{Stern School of Business, New York University, USA. \textbf{Email:} zzhou@stern.nyu.edu}
}
  
\maketitle
\begin{abstract}
In this paper, we consider online learning in generalized linear contextual bandits where rewards are not immediately observed. Instead, rewards are available to the decision maker only after some delay, which is unknown and stochastic. We study the performance of two well-known algorithms adapted to this delayed setting: one based on upper confidence bounds, and the other based on Thompson sampling. We describe  modifications on how these two algorithms should be adapted to handle delays and give regret characterizations for both algorithms. Our results contribute to the broad landscape of contextual bandits literature by establishing that both algorithms can be made to be robust to delays, thereby 
helping clarify and reaffirm the empirical success of these two algorithms, which are widely deployed in modern recommendation engines.
\end{abstract}

\section{Introduction}
The growing availability of user-specific data has welcomed the exciting era of personalized recommendation,
a paradigm that uncovers the heterogeneity across individuals and provides tailored service decisions that lead to improved outcomes. Such heterogeneity is ubiquitous across a variety of application domains (including online advertising, medical treatment assignment, product/news recommendation~(\cite{LCLS2010}, \cite{BCN2012},\cite{chapelle2014},\cite{bastani2015online},\cite{SBF2017})) and manifests itself as different individuals responding differently to the recommended items. Rising to this opportunity, contextual bandits~(\cite{besbes2009dynamic, rigollet2010nonparametric, goldenshluger2011note, hsu2014taming,agrawal2016efficient}) have emerged to be the predominant mathematical formalism that provides an elegant and powerful formulation: its three core components, the features (representing individual characteristics), the actions (representing the recommendation), and the rewards (representing the observed feedback), capture the salient aspects of the problem and provide fertile ground for developing algorithms that balance exploring and exploiting users' heterogeneity.

As such, the last decade has witnessed extensive research efforts in developing effective and efficient contextual bandits algorithms. In particular, two types of algorithms--upper confidence bounds (UCB) based algorithms (\cite{LCLS2010,FCGS2010, chu2011contextual,JBNW2017, LLZ2017}) and Thompson sampling (TS) based algorithms (\cite{AG2013a, AG2013b, RV2014, russo2016information,agrawal2017thompson})--stand out from this flourishing and fruitful line of work: their theoretical guarantees have been analyzed in many settings, often yielding (near-)optimal regret bounds; their empirical performance have been thoroughly validated, often providing insights into their practical efficacy (including the consensus that TS based algorithms, although sometimes suffering from intensive computation for posterior updates, are generally more effective than their UCB counterparts, whose performance can be sensitive to hyper-parameter tuning). To a large extent, these two family of algorithms have been widely deployed
in many modern recommendation engines.

However, a key assumption therein--both the algorithm design and their analyses--is that the reward is immediately available after an action is taken. Although useful as a first-step abstraction, this is a stringent requirement that is rarely satisfied in practice, particularly in large-scale systems where the time-scale of a single recommendation is significantly smaller than the time-scale of a user's feedback. For instance, in E-commerce, a recommendation is typically made by the engine in milliseconds, whereas a user's response time (i.e. to buy a product or conversion) is typically much larger, ranging from hours to days, sometimes even to weeks. For instance, a thorough empirical study in~\cite{chapelle2014} found that more than 10\% of the conversions in Criteo (a real-time bidding company) were at least 2 weeks old. Furthermore, \cite{chapelle2014} found that the delay distribution from the company's data follows the exponential distribution closely and hence does have heavy tails. Similarly, in clinical trials, it is infeasible to immediately observe and hence take into account the medical outcome after applying a treatment to a patient--collecting medical feedback can be a time-consuming and often random process; and in general, it is common to have applied trial treatments to a large number of patients, with individual medical outcomes only available much later at different, random points in time. In both the E-commerce (\cite{KCW2001,chapelle2014})and the clinical trials cases (\cite{CC2011}), a random and often significantly delayed reward is present. Further, such delays empirically often follow a heavy tail distribution, and hence \textit{a priori} can have substantially negative impact on the learning performance. Consequently, to understand such impact of delays,
adjustments in classical formulations must be made, both at the algorithmic level and at the analysis level.

\subsection{Related Work}  


In the past five years or so, the problem of learning on bandits with delays has received increasing attention and has been studied in several different settings in the existing literature, where most of the efforts have concentrated on the multi-armed bandits setting, including both the stochastic multi-armed bandits and the adversarial multi-armed bandits.

For stochastic multi-armed bandits with delays, \cite{JGS2013} show a regret bound $O(\log T +\E[\tau]+\sqrt{\log T \E[\tau]})$ where $\E[\tau]$ is the mean of the \textbf{iid} delays.
\cite{DKVB2014} consider Gaussian Process bandits with a bounded stochastic delay.
\cite{MLBP2015} follow the work of \cite{JGS2013} and propose a queue-based multi-armed bandit algorithm to handle delays. 
 \cite{PASG2017} match the same regret bound as in \cite{JGS2013} when feedback is not only delayed but also anonymous.

For adversarial multi-armed bandits with delays, \cite{NAGS2010} establish the regret bound of $\E [R_T] \leq O(\tau_{\text{const}})\times \E[R^{\prime}_T(\frac{T}{\tau_{\text{const}}})]$ for Markov decision process, where $\tau_\text{const}$ is the constant delay and $R^{\prime}_T$ is the regret without delays.  
\cite{CGM2019} consider adversarial bandits with fixed constant delays on the network graph, with a minimax regret of the order $\tilde{O}\sqrt{(K+\tau_{\text{const}})T}$, where $K$ is the number of arms.
Another related line of work to adversarial multi-armed bandits is adversarial learning with full information, where the rewards for all arms are observed. Different variants of this problems in the delayed setting have been studied by \cite{WO2002}, \cite{mesterharm2005}, \cite{QK2015} and \cite{GST2016}.

On the other hand, learning in contextual bandits with delays are much less explored.
\cite{JGS2013} consider learning on adversarial contextual bandits with delays and establish an expected regret bound $\E \left[R_T\right] \leq (1+\E[M_T^*])\times \E \left[R^{\prime}_T\left( \frac{T}{1+\E[M_T^*]}\right)\right]$ by using a black-box algorithm, where $M_T^*$ is the running maximum number of delays up to round $T$. \cite{DHKKLRZ2011} consider stochastic contextual bandits with a fixed constant delay. The reward model they consider is general (i.e. not necessarily parametric); however, they require the policy class to be finite. In particular, they obtain the regret bound $O(\sqrt{K\log N}(\tau_{\text{const}} +\sqrt{T}))$, where $N$ is the number of policies and $\tau_{\text{const}}$ is again the fixed constant delay.

Finally, we also note that there is a growing literature on offline contextual bandits (for a highly incomplete list, see~\cite{dudik2011doubly,swaminathan2015batch,athey2017efficient,zhou2018offline,kitagawa2018should,off-policy-evaluation-slate-recommendation,deep-learning-logged-bandit-feedback}). This is a setting where all the data has been collected upfront and a policy needs to be learned from this batch data at once. Although sharing the same primitives (contexts, actions and rewards), this problem has important differences from the online setting. In particular, the exploration part is missing in this problem and a separate set of challenges exist in the offline case. In this setting, delays would have no impact since 
all the rewards will have been collected at the end (except perhaps at the tail of the batch).



\subsection{Our Contributions}
In this paper, we consider learning on generalized linear (stochastic) contextual bandits with stochastic unbounded delays. Our contributions are two-fold.
First, we design two delay-adaptive algorithms for generalized linear contextual bandits, one based on UCB, the other based on TS. We refer to the two variants as Delayed UCB (DUCB, as given in Algorithm~\ref{DUCB-GLCB}) and Delayed TS (DTS, as given in Algorithm ~\ref{alg:PS-DGLM}) respectively.
DUCB requires a carefully designed delay-adaptive confidence parameter, which depends on how many rewards are missing up to the current time step. In contrast, DTS is a straightforward adaptation that incorporates the delayed rewards as they become available.

Second, we give regret characterizations of both DUCB and DTS under (1) independent stochastic, unbounded delays that can have heavy tails, (2) unbounded Markov delays that can have near-heavy tails (tails that are arbitrarily close to exponential tails), and (3) unbounded delays with any dependency structure that have light (sub-Gaussian) tails. In particular, as a special case of our results, when the delays are \textbf{iid} with mean $\mu_I$, we have a high-probability regret bound of $\tilde{O}\left(\left(\sigma_G\sqrt{d}+\mu_I d+d\right)\sqrt{T}\right)$ on DUCB, where $\sigma_G$ is a parameter characterizing the tail bound of the delays and $d$ is the feature dimension. For comparison, the state-of-the-art regret bound of UCB on generalized linear contextual bandits without delays is $\tilde{O}\left( d\sqrt{T}\right)$ (\cite{FCGS2010,LLZ2017}). For DTS, we have the Bayesian regret bound of  $\tilde{O}\left(\left(\sigma_G\sqrt{d}+ \mu_I \sqrt{d}+d\right)\sqrt{T}\right)$. 
For comparison, the state-of-the-art Bayesian regret bound of TS on generalized linear contextual bandits without delays is $\tilde{O}\left( d\sqrt{T}\right)$ (\cite{RV2014,russo2016information}).
The regret bounds we have obtained highlight the dependence on the delays in two ways: one is how much delay is present on average, the other is how heavy the tail of the distribution is.
Both factors contribute to the degradation of the regret bounds: that the average delay enlarges regret is intuitive; that the tail influences regret is because a more likely large delay (at the far right end of a tail) can delay the learning for that context significantly, particularly in the early stages when the decision maker is unsure about the underlying parameter is. 

To the best of our knowledge, these regret bounds provide the first theoretical characterizations in generalized linear contextual bandits with large delays.
Our results contribute to the broad landscape of contextual bandits literature by establishing that both algorithms are robust to delays, thereby 
helping clarify and reaffirm the empirical success of these two algorithms, which are widely deployed in modern recommendation engines.

	Some of the initial results have appeared in the conference version~\cite{zhou2019}. 
	Our work here provides a comprehensive treatment of learning in generalized linear contextual bandits with large delays that incorporates substantially more in-depth inquiries on several fronts. 
	First,  we consider the heavier-tailed delays that include exponential distributions whereas ~\cite{zhou2019} only dealt with light-tailed delays that are either sub-Gaussian or have ($1+q$)-th moment (for some $q > 0$).
	This relaxation is important both from an empirical standpoint and from a theoretical standpoint.  
	Empirically,  as mentioned earlier, the field study in~\cite{chapelle2014} found that the delay distribution from the company's data follows the exponential distribution closely, rather than a sub-Gaussian distribution that is commonly assumed in the bandits literature. Theoretically, establishing guarantees
	in this larger-delay regime requires us to develop a new (and arguably more elegant) argument from that in~\cite{zhou2019}, which is not applicable here. We explain the technical difficulty in more detail in Section \ref{sec:pre}.
	Second, the sole focus of
	~\cite{zhou2019} is on adapting and analyzing UCB-based algorithms. 
	However, as mentioned earlier, it is known that Thompson sampling often achieves superior empirical performance, despite the fact that their theoretical bounds (when no delays are present) may not match exactly those of the UCB algorithms. Furthermore, TS-based algorithms do not suffer from hyper-parameter tuning and can effectively incorporate prior and can therefore significantly outperform (when priors are available and correct). Consequently, in this paper, in addition to adapting and analyzing the UCB-based algorithms, we also discuss (in Section~\ref{sec:ts}) the adaptation of TS-based algorithms in the delayed feedback setting and obtain regret bounds that characterize the corresponding performance. Finally, we move beyond the regime of the independent delay setting studied in~\cite{zhou2019}, and instead consider (in Section~\ref{sec:extension}) the much more general and realistic history-dependent delays setting. We give regret bounds of both UCB-based algorithms and TS-based algorithms, under both the Markov delays assumption and the general stationary delays assumption.
	We also highlight, in this unified presentation, the comparison of the various regret bounds as the assumption on delays get progressively weakened.

\input{rest.tex}

\input{extension.tex}

\section{Conclusion}
A thorough empirical study by \cite{CL2011} shows superior performance of TS-based algorithms on stochastic contextual (and multi-armed) bandits  with delayed rewards. This matches the existing consensus that when there is no delays, TS-based algorithms tend to work better empirically than UCB-based algorithms, even though the regret of the latter is comparable to (and sometimes superior to) the former. In this delayed setting, we obtain comparable theoretical guarantees for DTS, and thus, together with the simplicity of the algorithm itself (i.e. no hyper-parameter tuning) further clarify why TS-based algorithms are more appealing choices in practice. 

\newpage
\bibliographystyle{plain}
\bibliography{refs}

\newpage
\appendix

\input{appendix.tex}

\end{document}

%% file: rest.tex
\section{Problem Setup}

In this section, we describe the formulation for learning in generalized linear contextual bandits (GLCB)
in the presence of delays. We start by reviewing the basics of generalized linear contextual bandits, followed by a description of the delay model. Before proceeding, we first fix some notation.

For a vector $x \in \R^d$, we use $\|x\|$ to denote its $l_2$-norm and $x^{\prime}$ its transpose. $\B^d:=\{x \in\R^d:\|x\| \leq 1\}$ is the unit ball centered at the origin. The weighted $l_2$-norm associated with a positive-definite matrix $A$ is defined by $\|x\|_A := \sqrt{x^{\prime}Ax}$. The minimum and maximum singular values of a matrix $A$ are written as $\lambda_{\min}(A)$ and $\|A\|$ respectively. 
For two symmetric matrices $A$ and $B$ the same dimensions, $A \succeq B$ means that A-B is positive semi-definite. For a real-valued function f, we use $\dot{f}$ and $\ddot{f}$ to denote its first and second derivatives. Finally, $[n]:=\{1,2,\cdots,n\}$.

 \subsection{Generalized Linear Contextual Bandits}
 
\paragraph{Decision procedure.} We consider the generalized linear contextual bandits problem with $K$ actions. At each round $t$, the agent observes a context consisting of a set of $K$ feature vectors $x_t:=\{x_{t,a} \in \R^d\vert a \in [K]\} $, which is drawn \textbf{iid} from an unknown distribution $\gamma$ with $\|x_{t,a}\| \leq 1$. Each feature vector $x_{t,a}$ is associated with an unknown stochastic reward $y_{t,a}\in[0,1]$. If the agent selects one action $a_t$,  there is a resulting reward $y_{t,a_t}\in[0,1]$ associated. In the standard contextual bandits setting, the reward is immediately observed after the decision is made and the observed reward can be utilized to make decision in the next round. 

Although it is generally understood in the contextual bandits literature, for completeness, here we briefly discuss the meaning of the above quantities, as well as where they come from. In general, at each round $t$, an individual characterized by $v_t$ (a list of characteristics associated with that individual) is drawn from a population and becomes available.
When the decision maker decides to apply action $a_t$ (one of the available $K$ actions) to this individual,
then a reward $y_t(v_t, a_t)$ is obtained: this reward can depend stochastically on both the individual characteristics $v_t$ and the selected action $a_t$. However, in practice, for both modelling and computational reasons, one often first featurizes the individual characteristics and the actions.
In particular, with sufficient generality, one assumes  $\mathbf{E}[y_t(v_t, a_t) \mid v_t, a_t] = g_{\theta} (\phi(v_t, a_t))$,
where $g_{\theta}(\cdot)$ is the parametrized \textbf{mean} reward function and $\phi(v_t, a_t)$ extracts the features from the given raw individual characteristics $v_t$ and action $a_t$. In the above formulation,
as is standard in the contextual bandits literature, we assume the feature map $\phi(\cdot)$ is known and given and $x_{t,a} = \phi(v_t, a)$. If $V_t$ is already a vector in Euclidean space, then a common choice for the feature extractor is $\phi(v_t, a) = [\mathbf{0}, \dots, \mathbf{0}, v_t, \mathbf{0}, \dots, \mathbf{0}]$: that is, a $Kd$-dimensional vector with all zeros except at the $a$-th block.

\paragraph{Relationship between reward $Y$ and context $X$.} In terms of the relationship between $Y_{t,a}$ and $X_{t,a}$, we follow the standard generalized linear contextual bandits literature (\cite{FCGS2010,LLZ2017}).
Define $\HH^0_t = \{(s,x_s,a_s,y_{s,a_s}), s \leq t-1\} \cup \{x_t\}$ as the information available at the beginning of round $t$. The agent maximizes the cumulative expected rewards over $T$ rounds with information $\HH^0_t$ at each round $t$ ($ t \geq 1$). Suppose the agent takes action $a_t$ at round $t$. Denote by $X_t =x_{t,a_t}$, $Y_t = y_{t,a_t}$ and we assume the conditional distribution of $Y_t$ given $X_t$ is from the exponential family. Therefore its density is given by
\begin{eqnarray}\label{glm}
\P_{\theta^*}(Y_t|X_t) = \exp \left( \frac{Y_tX_t^{\prime}\theta^*-m(X_t^{\prime}\theta^*)}{h(\eta)}+A(Y_t,\eta)\right).
\end{eqnarray}
Here, $\theta^*$ is an unknown number under the frequentist setting; 
$\eta \in \R^+$ is a given parameter; $A$, $m$ and $h$ are three normalization functions mapping from $\R$ to $\R$.

 For exponential families, $m$ is infinitely differentiable, $\dot{m}(X^{\prime}\theta^{*})=\E[Y \vert X]$, and $\ddot{m}(X^{\prime}\theta^{*})=\V(Y \vert X)$. Denote $g(X^{\prime}\theta^{*}) =\E[Y \vert X] $ , one can easily verify that $g(x^{\prime}\theta)=x^{\prime}\theta$ for linear model, $g(x^{\prime}\theta)= \frac{1}{1+\exp(-x^{\prime}\theta)}$ for logistic model and $g(x^{\prime}\theta) = \exp (x^{\prime}\theta)$ for Poisson model. In the generalized linear model (GLM) literature (\cite{NW1972,McCullagh2018}), $g$ is often referred to as the {\it inverse link function}. 

Note that \eqref{glm} can be rewritten as the GLCB form,
\begin{eqnarray}\label{glb}
Y_t = g( X_t^{\prime}\theta^*) + \epsilon_t,
\end{eqnarray}
where $\{\epsilon_t, t\in [T]\}$ are independent zero-mean noise, $\HH^0_t$-measurable with $\E[\epsilon_t|{\HH^0_{t}}]=0$. Data generated from \eqref{glm} automatically satisfies the sub-Gaussian condition:
\begin{eqnarray}\label{subgaussian}
\E\left[\exp({\lambda \epsilon_t})|{\HH^0_{t}}\right] \leq \exp \left({\frac{\lambda^2 \hat{\sigma}^2}{2}}\right).
\end{eqnarray}
Throughout the paper, we denote  $\hat{\sigma}>0$ as the sub-Gaussian parameter of the noise $\epsilon_t$. 

\begin{remark}
 In this paper, we focus on the GLM with exponential family \eqref{glm}. In general, one can work with model \eqref{glb} under the sub-Gaussian assumption \eqref{subgaussian}. Our analysis will still hold by considering maximum quasi-likelihood estimator for \eqref{glb}.
See more explanations in Section \ref{app:MLE}.
\end{remark}


 \subsection{The Delay Model}\label{sec:delay_model}
Unlike the traditional setting where each reward is immediately observed, 
here we consider the case where stochastic and unbounded {\it delays} are present in revealing the rewards.  
Let $T$ be the number of total rounds. At round $t$, after the agent takes action $a_t$, the reward $y_{t,a_t}$ may not be available immediately. Instead, it will be observed at the end of round $t+D_t$ where $D_t$ is the delay at time $t$. We assume $D_t$ is a non-negative random number which is independent of $\{D_s\}_{s \leq t-1}$ and  $\{x_s, y_{s,a_s}, a_s\}_{s \leq t}$. 
First, we define the available information for the agent at each round.

\paragraph{Information structure under delays.}
At any round $t$, if $D_s+s\leq t-1$ (reward occurred in round $s$ is available at the beginning of round $t$), then we call $(s,x_s,y_{s,a_s},a_s)$  the {\it complete information tuple} at round $t$. If $D_s+s\geq t$, we call $(s,x_s,a_s)$ the {\it incomplete information tuple} at the beginning of round $t$. Define
$$\HH_t = \left\{(s,x_s,y_{s,a_s},a_s)\,\,\vert\,\, s+D_s \leq  t-1 \right\}\cup \left\{(s,x_s,a_s)\,\,\vert\,\, s\leq t-1, s+D_s \geq  t \right\}\cup \left\{x_t \right\},$$
then $\HH_t$ is the information (filtration) available at the beginning of round $t$ for the agent to choose action $a_t$. In other words, $\HH_t$ contains all the incomplete and complete information tuples up to round $t-1$ and the content vector $x_t$ at round $t$.

Moreover define
\begin{eqnarray}\label{after_info}
\FF_t = \{ (s,x_s,a_{s},y_{s,a_s}) \,\,\vert\,\, s+ D_s \leq t\}.
\end{eqnarray}
Then $\FF_t$ contains all the complete information tuples $ (s,x_s,a_{s},y_{s,a_s})$ up to the end of round $t$.
Denote $\II_{t}=\FF_{t}-\FF_{t-1}$, $\II_t$ is the new complete information tuples revealed at the end of round $t$.

\paragraph{Performance criterion.} 
Under the frequentist setting, assume there exists an unknown true parameter $\theta^* \in \R^d$. The agent's strategy can be evaluated by comparing her rewards to the best reward. To do so, define the optimal action at round $t$ by $a_t^*=\arg \max_{a \in [K]} g(x_{t,a}^{\prime} \theta^*)$. Then, the agent's total regret of following strategy $\pi$ can be expressed as follows
\[R_T(\pi):= \sum_{t=1}^T \left(g \left(x_{t,a^*_t}^{\prime}\theta^*\right)-g\left( x_{t,a_t}^{\prime}\theta^*\right)\right),\]
where $a_t\sim \pi_t$ and policy $\pi_t$ maps $\HH_t$ to the probability simplex $\Delta^K:=\{(p_1,\cdots,p_K)\,\,\vert\,\,\sum_{i=1}^K p_i=1, p_i \geq 0\}$. Note that $R_T(\pi)$ is in general a random variable due to the possible randomness in $\pi$. 

%

\paragraph{Assumptions.}
Throughout the paper, we assume the following assumption on distribution $\gamma$ and function $g$, which is standard in the generalized linear bandit literature (\cite{FCGS2010,LLZ2017,JBNW2017}).
\begin{assumption}[GLCB]\label{a1}
\begin{itemize}
 \item $\lambda_{\min} (\E[\frac{1}{K} \sum_{a \in [K]}x_{t,a}x_{t,a}^{\prime}]) \geq \sigma_0^2$ for all $t \in [T]$.
\item $\kappa := \inf_{\{\|x\|\leq 1, \|\theta-\theta^*\|\leq 1\}}\dot{g}(x^{\prime}\theta)>0$.
\item $g$ is twice differentiable. $\dot{g}$ and $\ddot{g}$ are upper bounded by $L_{g}$ and $M_{g}$, respectively.
\end{itemize}
 \end{assumption}

In addition, we assume the delay sequence $\{D_t\}_{t=1}^T$ satisfies the following assumption.

\begin{assumption}[Delay]\label{delay}
Assume $\{D_t\}_{t=1}^T$ are independent non-negative random variables with {\it tail-envelope distribution} $(\xi,\mu,M)$. That is,
there exists a constant $M>0$ and a distribution $\xi$ with mean $\mu < \infty$ such that
for any  $m \geq M$ and  $t \in [T]$,
\[
\P(D_t \geq m) \leq \P(D \geq m),
\]
where  $D \sim \xi$.
Furthermore, assume there exists $q\geq 0$ such that
\[
\P(D-\mu \geq x) \leq \exp \left(\frac{-x^{1+q}}{2\sigma^2}\right),
\]
where $\E [D]=\mu$.
\end{assumption}

Assumption \ref{delay} includes the most common delay patterns in real-world applications. $D$ is sub-Gaussian when $q=1$ and $D$ has exponential delays when $q=0$.
When $D_t$'s are  \textbf{iid}, the following condition guarantees Assumption \ref{delay}:
\[
\P(D_t-\E[D_t] \geq x) \leq \exp\left(\frac{-x^{1+q}}{2\tilde{\sigma}^2}\right),
\]
with some $\tilde{\sigma}>0$ and $q \ge 0$. We summarize the parameter definition in Table~\ref{tab:parameters}. (See Section \ref{app:table}.)

Note that with Assumption \ref{delay}, we do not need to assume all delays have identical distributions, as long as they are independent over time. 
Since there exists an envelop distribution $\xi$ uniformly dominating the tail probability of all delays, we can
get a handle on the tail of all the delay distributions.
This can be viewed as the regularity condition on the delays.

\section{Delayed Upper Confidence Bound (DUCB) for GLCB}\label{sec:UCB}
In this section, we propose a UCB type of algorithm for GLCB adapting the delay information in an online version. Let us first introduce the maximum likelihood estimator we adopt and then state the main algorithm.

\subsection{Maximum Likelihood Estimators (MLEs).}\label{app:MLE}
Denote $T_t = \{s:s\leq t-1,D_s+s\leq t-1\}$
as the set containing timestamps with complete information tuples at the beginning of round $t$. We use data with timestamps in $T_t$ to construct the MLE. Suppose we have independent samples of $\{Y_s : s \in T_t\}$ condition on $\{X_s: s \in T_t\}$. The log-likelihood function of $\theta$ under \eqref{glm} is
\begin{eqnarray*}
\log l \left(\theta\,\,\vert \,\,T_t \right) &=& \sum_{s \in T_t} \left[ \frac{Y_s X_s^{\prime}\theta-m(X_s^{\prime}\theta)}{v(\eta)}+B(Y_s,\eta) \right]\\
&=& \frac{1}{v(\eta)} \sum_{s \in T_t} \left[ Y_s X_s^{\prime}\theta-m(X_s^{\prime}\theta)\right] +\text{constant}.
\end{eqnarray*}
Therefore, the MLE can be defined as
$$
\hat{\theta}_t \in \arg\max_{\theta \in \Theta} \sum_{s \in T_t} \left[ Y_s X_s^{\prime}\theta-m(X_s^{\prime}\theta) \right].$$
Since $m$ is differentiable with $\ddot{m}\geq0$, the MLE can be written as the solution of the following equation
\begin{eqnarray}\label{mle_final}
\sum_{s \in T_t} (Y_s-g(X_s^{\prime}\theta)) X_s= 0,
\end{eqnarray}
which is the estimator we use in Step 4 of Algorithm \ref{DUCB-GLCB}.

Note that, the general GLCB, a semi-parametric version of the GLM, is obtained by assuming only that $\E[Y|X]=g(X^{\prime}\theta^*)$ (see \eqref{glb}) without further assumptions on the conditional distribution of $Y$ given $X$. In this case, the estimator obtained by solving \eqref{mle_final} is referred to as the {\it maximum quasi-likelihood estimator}.
It is well-documented that this estimator is consistent under very general assumptions as long as matrix $\sum_{s\in T_t}X_sX_s^{\prime}$
tends to infinity as $t\rightarrow \infty$ (\cite{CHY1999,FCGS2010}).

\subsection{Algorithm: DUCB-GLCB}\label{subsec: UCBalgorithm}
Denote $G_t = \sum_{s=1}^{t-1} \I\{s+ D_s \geq t\}$ as the number of missing reward when the agent is making a prediction at round $t$. Further denote $W_t =\sum_{s \in T_{t}} X_s X_s^{\prime} $ as the matrix consisting feature information with timestamps in $T_{t}$ and $V_t = \sum_{s=1}^{t-1} X_s X_s^{\prime}$ as the matrix consisting all available features at the end of round $t-1$. Then the main algorithm is defined as follows.
\begin{algorithm}[H]
  \caption{\textbf{DUCB-GLCB}}
  \label{DUCB-GLCB}
\begin{algorithmic}[1]
  \STATE \textbf{Input}: the total rounds $T$ , model parameters $d$ and $\kappa$, and tuning parameters $\tau$ and $\delta$.
  \STATE \textbf{Initialization}: randomly choose $\alpha_t \in [K]$ for $t \in [\tau]$, set $V_{\tau+1}=\sum_{i=1}^{\tau}X_s X_s^{\prime}$, $T_{\tau+1}:=\{s\,:\, s\leq \tau, s+D_s \leq \tau\}$, $G_{\tau+1} =\tau- |T_{\tau+1}|$ and $W_{\tau+1}=\sum_{s \in T_{\tau+1}}X_s X_s^{\prime}$
 \FOR {$t = \tau+1,\tau+2,\cdots,T$}
 \STATE {\bf Update Statistics}: calculate the MLE $\hat{\theta}_t$ by solving $\sum_{s \in T_{t}} (Y_s-g(X_s^{\prime}\theta))X_s=0$
 \STATE {\bf Update Parameter}: $\beta_t = \frac{\hat{\sigma}}{\kappa} \sqrt{\frac{d}{2} \log \left(1+\frac{2(t-G_t)}{d}\right)+\log(\frac{1}{\delta})}+{\sqrt{G_t}}$
 \STATE {\bf Select Action}: choose $a_t = \arg \max_{a \in [K]} \left( x_{t,a}^{\prime} \hat{\theta}_t +\beta_t \| x_{t,a}\|_{V_t^{-1}}\right)$
 \STATE {\bf Update Observations}: $X_t \leftarrow x_{t,a_t}$, $V_{t+1} \leftarrow V_t +X_t X_t^{\prime}$ and  $T_{t+1} \leftarrow T_{t} \cup\{s \,:\, s+D_s=t\}$, $G_{t+1} = t-|T_{t+1}|$, and $W_{\tau+1}=W_{\tau}+\sum_{s: s+D_s=t}X_s X_s^{\prime}$
\ENDFOR
\end{algorithmic}
\end{algorithm}

\begin{remark}[Comparison to UCB-GLM Algorithm in \cite{LLZ2017}]We make several adjustments to the UCB-GLM Algorithm in \cite{LLZ2017}. First, in step 4 (statistics update), we only use data with timestamps in $T_t$ to calculate the estimator using MLE.
In this step, using data without reward will cause bias in the estimation. Second, when selecting the action in step 5, parameter $\beta_t$ is updated adaptively at each round whereas in \cite{LLZ2017}, the corresponding parameter is constant over time.
Moreover, in step 4, we choose to use $V_t$ to normalize the context vector $X_{t,a}$ instead of $W_t$.
\end{remark}

\subsection{Preliminary Analysis}\label{sec:pre}
Denote $G_t^* =\max_{1\leq s \leq t} G_s$ as the running maximum number of missing reward up to round $t$. 
The property of $G_t$ and $G^*_{t}$ is the key to analyze the regret bound for both UCB and Thompson sampling algorithms. We next characterize the tail behavior of $G_t$ and $G^*_t$.

\begin{proposition}[Properties of $G_t$ and $G_t^\star$]\label{prop:delay}
Assume Assumption \ref{delay}. Denote $\sigma_G =  \sigma \sqrt{2+q} $. Then,
\begin{enumerate}
\item $G_t$ is sub-Gaussian. Moreover, for all $t\geq 1$, with probability $1-\delta$
\begin{eqnarray}\label{Gt}
{G}_t \leq 2(\mu+M) +\sigma_G\sqrt{2\log\left(\frac{1}{\delta}\right)}+2\sigma_G^2\log C_3+1,
\end{eqnarray}
where $C_3 = 2\sigma^2+1$.
\item With probability $1-\delta$,
\begin{eqnarray}\label{Gt_max}
{G}_T^* &\leq& 2(\mu+M) +\sigma_G \sqrt{2\log T}+ 2\sigma_G^2\log C_3 \nonumber \\
&& +\sigma_G \sqrt{2\log\left( \frac{1}{\delta}\right)+2\log C_3 \sigma_G \sqrt{2\log T}+2\log C_3}+1,
\end{eqnarray}
where $G_T^* =\max_{1\leq s \leq T} G_s$.
\item Define $W_t = \sum_{s \in T_t}X_s X_s^{\prime}$ where $X_t$ is drawn \textbf{iid}. from some distribution $\gamma$ with support in the unit ball $\mathbb{B}_d$. Furthermore, let $\Sigma := \mathbb{E}[X_t X_t^{\prime}]$ be the second moment matrix, and $B$ and $\delta >0$ be two positive constants. Then there exist positive, universal constants $C_1$ and $C_2$ such that $\lambda_{\min}(W_t)\geq B$ with probability at least $1-2\delta$, as long as 
\begin{eqnarray}\label{number_regularity}
t \geq \left( \frac{C_1\sqrt{d}+C_2\sqrt{\log(\frac{1}{\delta})}}{\lambda_{\min}(\Sigma)}\right)^2 +\frac{2B}{\lambda_{\min}(\Sigma)}+  2(\mu+M) +\sigma_G\sqrt{2\log\left(\frac{1}{\delta}\right)}+2\sigma_G^2\log C_3+1.
\end{eqnarray}
\end{enumerate}
\end{proposition}

A special case of Proposition \ref{prop:delay}-1 is when $D_i$'s are \textbf{iid} and $q=0$. Now assume $D_i\sim D$ are {\bf{iid}} with exponential-decays:
\begin{eqnarray}\label{special_delay_case}
\P(D- \mu_I \geq t)\leq \exp(-\frac{t}{2\sigma_{I}^2}),
\end{eqnarray}
and $ \mu_I = \E D$. Then  with probability $1-\delta$, we have
\begin{eqnarray}\label{special_G_t}
G_t-\mu_I \leq  2\sigma_{I}\sqrt{\log \left(\frac{1}{\delta}\right)} +1+4\sigma_{I}^2\log(2\sigma_{I}^2).
\end{eqnarray}

At a high level, the proof utilizes the fact that, with high probability, there will be a lot of zero terms in the summation $G_t = \sum_{s=1}^{t-1}\I(s+D_s\geq s)$ when $t$ is large. This is done by designing a sequence of stopping times for the successes.
We highlight the idea by showing result \eqref{special_G_t} for the special case when $D_t$'s are \textbf{iid} and $q=0$. The full version of the proof is deferred to Appendix \ref{proof}.

\begin{proof}[Sketch of the proof.]

Define $V=\sum_{i=1}^{\infty}\I(D_i-\mu_I \geq i)$ where $D_i\sim D$ are {\bf{iid}} that satisfies \eqref{special_delay_case}.

Now let us define the following sequence of stopping times, $(k \geq 1)$,
\[
T(k) = \inf \{t >T(k-1):D_t \geq t\},
\]
where $T(k)$ is the time of the $k^{\text th}$ success. Therefore,
\begin{eqnarray}
\P(V \geq j) &=& \P(T(1)<\infty,T(2)<\infty,\cdots,T(j-1)<\infty,T(j)<\infty)\nonumber\\
&=& \Pi_{k=1}^{j}\P\left(T(k)<\infty \vert T(i)<\infty \,\,\text{ for }\,\, i \leq k-1\right)\label{tower}\\
&=& \Pi_{k=2}^{j}\P\left(T(k)<\infty \vert T(k-1)<\infty \right)\P\left(T(1)<\infty \right) \label{equal}\\
&\leq& \Pi_{k=1}^{j} \left(\sum_{i=k}^{\infty}\exp \left(-\frac{i}{2\sigma_{I}^2}\right)\right)\label{worst_case}\\
&\leq&\Pi_{k=1}^{j}  \left(2\sigma_{I}^{2}\exp \left(-\frac{k-1}{2\sigma_{I}^2}\right)\right)\label{integration}\\
&=& (2\sigma_I^2)^{j}\exp\left(-\frac{(j-1)j}{4\sigma_{I}^2}\right)\label{last}
\end{eqnarray}
\eqref{tower} holds by tower property. \eqref{equal} holds since event $\{T(k)<\infty \vert T(k-1)<\infty\}$ is equivalent to event $\{T(k)<\infty \vert T(j)<\infty \,\,{\text for}\,\, j \leq k-1\}$. Condition on $ T(k-1)<\infty$, we have $\P\left(T(k)<\infty \vert T(k-1)<\infty \right) \leq \P(\cup_{j\geq k}\I(D_j \geq j)) \leq \sum_{i=k}^{\infty}\exp \left(-\frac{i}{2\sigma_{I}^2}\right)$. The last inequality holds by the union bound. Therefore \eqref{worst_case} holds. Finally, \ref{integration} holds by integration.

Given \eqref{integration}, $V$ is sub-Gaussian and with probability $1-\delta$,
\[
V \leq 2\sigma_{I}\sqrt{\log \left(\frac{1}{\delta}\right)} +1+4\sigma_{I}^2\log(2\sigma_{I}^2).
\]
Similarly, we can show that, for any $t \ge 1$, $G_t$ is sub-Gaussian. With probability $1-\delta$, we have
\[
G_t-\mu_I \leq  2\sigma_{I}\sqrt{\log \left(\frac{1}{\delta}\right)} +1+4\sigma_{I}^2\log(2\sigma_{I}^2).
\]
\end{proof}

Note that $G_t$ is sub-Gaussian even when $D$ has near-heavy-tail distribution ($p \in [0,1)$).

\remark{
The proof of Proposition \ref{prop:delay} is simple but essential. It fully utilizes the property that the sequence in $V$ has a lot of zero terms (with high probability). 
In particular, one will not be able to fully obtain the result if one uses the standard approach and directly works at the level of ``the-sum-of-sub-Guassians-is-sub-Gaussian" and thereafter analyzing sum of sub-Gaussian constants, which is the method used in~\cite{zhou2019}. In order to drive this point home, we provide an approach in this direction using Hoeffding bound (Theorem \ref{thm9}). See Appendix \ref{further_G}. With such a approach, one can only handle the case when $q>0$, which excludes the most difficult scenario with exponential delays. With Hoeffding bound, the sub-Gaussian parameter for $V$ is of the form $\sigma = \sqrt{\sum_{i=1}^{\infty}\sigma_i^2}$ where  $\sigma$ is the sub-Gaussian parameter for indicator function $\I(G_i \geq i)$. Intuitively speaking, this Hoeffding bound does not take into consideration of the sparsity in the sequence. Therefore, the argument cannot reach the limit for $q=0$.
}

\subsection{Regret Bounds}\label{subsec: UCBregret}

\begin{theorem}\label{thm: UCBregret}
Assume Assumptions \ref{a1}-\ref{delay}.  Fix any $\delta$. There exists a universal constant $$C:=C(C_1,C_2,M,\mu,\sigma_0,\hat{\sigma},\,\sigma,\kappa)>0,$$ such that if we run DUCB-GLCB with $\tau := C \left(d+\log(\frac{1}{\delta})\right)$ and $\beta_t = \frac{\hat{\sigma}}{\kappa} \sqrt{\frac{d}{2} \log \left(1+\frac{2(t-G_t)}{d}\right)+\log(\frac{1}{\delta})}+G_t$, then, with probability at least $1-5\delta$, the regret of the algorithm is upper bounded by
\begin{eqnarray}\label{Rt}
R_T &\leq&\tau+ L_{g} \left[ 4\sqrt{\mu+M} \sqrt{T d \log \left(\frac{T}{d}\right)} + 2^{7/4}\sqrt{\sigma_G}(\log T )^{1/4} \sqrt{d \log \left(\frac{T}{d}\right)T}+\frac{2d \hat{\sigma}}{\kappa} \log \left( \frac{T}{d \delta}\right)\sqrt{T} \right. \nonumber \\
&& \,\,+\left.2\sqrt{2T d \log \left(\frac{T}{d}\right)}\left(\sqrt{\sigma_G } \left({2\log\left( \frac{1}{\delta}\right)+2\log C_3 \sigma_G \sqrt{2\log T}+2\log C_3}\right)^{1/4}\right.\right.\nonumber\\
&&\,\,\left.\left.+\sqrt{1+2\sigma_G^2\log C_3}\right)\right]
\end{eqnarray} \end{theorem}
For parameter definition, we refer to Table~\ref{tab:parameters} in Section \ref{app:table}.

The proof of Theorem \ref{thm: UCBregret} consists of three steps. The first step is to construct a confidence ball associated with the adaptive parameter $\beta_t$ and show that the true parameter falls into the confidence ball with high probability. The second step is to upper bound the normalized context sequence $\sum_{t=\tau+1}^{\tau+n}\|X_t\|_{V_t^{-1}}$. And the last step is to utilize the property of $G_t$ and $G^*_t$ proved in Proposition \ref{prop:delay}.
The details is deferred to Appendix \ref{proof}.

Given the high probability bound in Theorem \ref{thm: UCBregret}, one can show the expected regret bound without much of work. 

\begin{corollary}[Expected regret]
Assume Assumptions \ref{a1}-\ref{delay}. The expected regret is bounded by
\begin{eqnarray}\label{expected}
\E[R_T] = {O \left(d\sqrt{T}\log(T) +\sqrt{\sigma_G}\sqrt{Td}(\log(T))^{3/4}+(\sqrt{\mu+M}+\sigma_G) \sqrt{T d \log \left({T}\right)}\right).}
\end{eqnarray}
\end{corollary}
Given the result in \eqref{Rt}, \eqref{expected} holds by choosing $\delta=\frac{1}{T}$ and using the fact that $R_T \leq T$. 


{
The highest order term $O(d\sqrt{T}\log(T))$ does not depend on delays. This result is in line with the non-contextual stochastic bandit literature (\cite{JGS2013}). Delay impacts the expected regret bound in two folds. First, the sub-Gaussian parameter $\sigma_G$ and the mean-related parameter ${\mu+M}$  appears in the second-highest order term.
Second, the sub-Gaussian parameter ${\sigma_G}$ appears in the third-order term. Note that here we include the log factors in deciding the highest order term, the second highest order term and so on. If we exclude the log terms, then both delay parameters impact the regret bound multiplicatively. 
}


\subsection{Tighter Regret Bounds for Special Cases}
 When the sequence $\{D_s\}_{s=1}^T$ satisfies some specific assumptions, we are able to provide tighter high probability bounds on the regret. 
\begin{proposition}\label{prop:Gt} Given Assumptions \ref{a1}-\ref{delay},
we have the following results.
\begin{enumerate}
\item If there exists a constant $D_{\max}>0$ such that $\mathbb{P}(D_s \leq D_{\max})=1$ for all $s \in [T]$. Fix $\delta$. There exists a universal constant $C>0$ such that by taking $\tau =D_{\max}+C(d+\log(\frac{1}{\delta}))$, with probability $1-3\delta$,
the regret of the algorithm is upper bounded by
\begin{eqnarray}\label{bounded_delay_result}
R_T \leq  \tau + L_{g} \left( 2{\sqrt{D_{\max}}} \sqrt{2T d \log \left(\frac{T}{d}\right)} + \frac{2d \hat{\sigma}}{\kappa} \log \left(\frac{T}{d\delta}\right)\sqrt{T}\right).
\end{eqnarray}
\item Assume $D_1,\cdots,D_T$ are \textbf{iid} non-negative random variables with mean $\mu_I$. There exists $C>0$ such that by taking $\tau := C \left(d+\log(\frac{1}{\delta})\right)$, with probability $1-5\delta$, the regret of the algorithm is upper bounded by
\begin{eqnarray*}
R_T  \leq && \tau+ L_{g} \left[ {4\sqrt{\mu_I}} \sqrt{T d \log \left(\frac{T}{d}\right)} + {2^{7/4}\sqrt{\sigma_G}(\log T )^{1/4}} \sqrt{d \log \left(\frac{T}{d}\right)T}+\frac{2d \hat{\sigma}}{\kappa} \log \left( \frac{T}{d \delta}\right)\sqrt{T} \right. \\
&& \,\,+\left.2\sqrt{2T d \log \left(\frac{T}{d}\right)}\left(\sqrt{\sigma_G } \left({2\log\left( \frac{1}{\delta}\right)+2\log C_3 \sigma_G \sqrt{2\log T}+2\log C_3}\right)^{1/4}\right.\right.\\
&&\,\,\left.\left.+\sqrt{1+2\sigma_G^2\log C_3}\right)\right]
\end{eqnarray*}
\end{enumerate}
\end{proposition}

{When delays $\{D_s\}_{s=1}^T$ are bounded by $D_{\max}$, the delay paramter $D_{\max}$ only appears in the term $\sqrt{ Td \log{{T}}}$ and does not affect the highest order term $ d \log(\frac{T}{d\delta})\sqrt{T}$.  Compared to \eqref{expected}, there is no regret term on the order of $O(\sqrt{Td}\left(\log(T)\right)^{3/4})$ in \eqref{bounded_delay_result}. This is because we can provide a smaller number on the right hand side of \eqref{number_regularity} when delays are bounded.
When delays are \textbf{iid}, $\mu+M$ is replaced by $\mu_I$,  which is the common expectation of all the random delays. }

We refer to Appendix \ref{proof} for the proof of Proposition \ref{prop:Gt}.

 \section{Delayed Thompson Sampling (DTS) for GLCB}\label{sec:ts}
In section \ref{sec:UCB}, under the frequentist set-up, we assume there exists a true parameter $\theta^*$ and use UCB to encourage exploration and construct the confidence interval for $\theta^*$.
On the contrary, posterior sampling does not make use of upper confidence bounds to encourage exploration and instead relies on randomization. In this section, we operate in the Bayesian decision making setting and assume the decision maker is equipped with a prior distribution on $\theta^*$.
In this setting, the standard performance metric is Bayesian regret, defined as follows:
\[
R^B_T(\pi) = \E_{\theta^*,x}[R_T(\pi,\theta^*)]=\sum_{t=1}^T \mathbb{E}_{\theta^*,x} \left[g\left(x_{t,a^*_t(\theta^*)}^{\prime}\theta^*\right)-g\left(\ x_{t,a_t}^{\prime}\theta^*\right)\right],
\]
where $a_t\sim \pi_t$. 
Next, we present the Thompson sampling algorithm when adapted to the delayed setting.
 Algorithm~\ref{alg:PS-DGLM} provides a formal description.

\begin{algorithm}[H]
  \caption{\textbf{DTS-GLCB}}
  \label{alg:PS-DGLM}
\begin{algorithmic}[1]
  \STATE \textbf{Input}: the total rounds $T$ , tuning parameter $\tau$, prior $Q_0$
  \STATE \textbf{Initialization}: randomly choose $\alpha_t \in [K]$ for $t \in [\tau]$\\
  \STATE {\bf Update information}: $\FF_{\tau}$ according to \eqref{after_info}.
        \IF{$\II_{\tau}=\emptyset$}
        \STATE $Q_{1}(\theta) = Q_{0}(\theta)$
        \ELSE
       \STATE $Q_{1}(\theta) \propto Q_{0}(\theta \vert \FF_{\tau})\Pi_{(s,x_s,a_s,y_{s,a_s})\in \II_{\tau}} \P(y_{s,a_s}\vert \theta, x_{s,a_s})$
\ENDIF
 \FOR {$t = 1,2,\cdots,T-\tau$}
 \STATE {\bf Sample Model}: $\hat{\theta}_{t+\tau} \sim Q_t$ 
 \STATE {\bf Select Action}: $\bar{a}_{t+\tau} \in \arg\max_{a \in [K]} \left\langle x_{t+\tau,a},\hat{\theta}_{t+\tau}\right\rangle$
 \STATE {\bf Update information}: $\FF_{t+\tau}$ according to \eqref{after_info}. Define $\II_{t+\tau}:=\FF_{t+\tau}-\FF_{t+\tau-1}$ as the new information at round $t+\tau$
        \IF{$\II_{t+\tau+1}=\emptyset$}
        \STATE $Q_{t+1}(\theta) = Q_{t}(\theta)$
        \ELSE
       \STATE $Q_{t+1}(\theta) \propto Q_{t}(\theta \vert \FF_{\tau+t})\Pi_{(s,x_s,a_s,y_{s,a_s})\in \II_{t+\tau+1}} \P(y_{s,a_s}\vert \theta, x_{s,a_s})$
        \ENDIF
\ENDFOR
\end{algorithmic}
\end{algorithm}

\begin{remark}
Note that in Algorithm \ref{alg:PS-DGLM}, there is an exploration period of length $\tau$. The posterior distribution employed at round $\tau +1$ is conditioned on observations made over the
first $\tau$ time rounds. 
Another point to note is that Algorithm \ref{alg:PS-DGLM} is kept at an abstract level.
The exact computation depends on the prior chosen and the exponential family.
Note that every exponential family has a conjugate prior (\cite{DY1979}), which admits efficient posterior update. Section \ref{concrete} provides a concrete example on linear contextual bandits, which is a simple special case. We use this special case to illustrate how one can perform efficient incremental update in the presence of delays.
\end{remark}

\subsection{Delayed Thompson Sampling For Linear Contextual Bandits}\label{concrete}

When $g(x)=x$ and $m(x)=\frac{x^2}{2}$, \eqref{glm} reduces to 
\begin{eqnarray}\label{linear}
\P(Y|X) = \exp \left( \frac{YX^{\prime}\theta^*-(X^{\prime}\theta^*)^2/2}{h(\eta)}+A(Y,\eta)\right).
\end{eqnarray}

Recall, from Bayes' theorem, the posterior distribution is equal to the product of the likelihood function $\theta \rightarrow \P(y\vert \theta)$ and prior $\P(\theta)$, normalized by the probability of the data $\P(y)$:
\[
\P(\theta\vert y) =\frac{\P(y\vert \theta)\P(\theta)}{\int\P(y\vert\theta)\P(\theta^{\prime})d\theta^{\prime}}
\]
Different choices of the prior distribution $\P(\theta)$ may make the integral more or less difficult to calculate. Moreover, the product $\P(y\vert\theta)\P(\theta)$ may take one form or another. But for certain choices of the prior, the posterior will have the same form as the prior, with possibly different parameter values. Such a choice is a {\it conjugate prior}. The conjugate prior, giving a closed-form expression for the posterior, makes Thompson sampling efficient to update. Further notice that, every exponential family has a conjugate prior (\cite{DY1979}). 

Now we consider the normal conjugate prior for the linear model \eqref{linear}.
Let $B_t = aI_d +\sum_{s\in T_{t}} x_{s,a_s}x_{s,a_s}^{\prime}$ and $\theta_t = B_t^{-1}\left(\sum_{s\in T_{t}} x_{s,a_s}y_{s,a_s} \right)$. Given the linear model \eqref{linear}, suppose we have $Y|X$ is Gaussian with $\NN(X^{\prime}{\theta},v^2)$. If the {\bf prior} for $\theta$ at round $t$ is given by $\NN (\theta_t,v^2B_t^{-1})$, then it is easy to verify that the {\bf posterior} distribution at around $t+1$ is $\NN (\theta_{t+1},v^2B_{t+1}^{-1})$. Then Algorithm \ref{alg:PS-DGLM} becomes
\begin{algorithm}[H]
  \caption{\textbf{DTS-LCB}}
  \label{alg:PS-DGLM-Lin}
\begin{algorithmic}[1]
  \STATE \textbf{Input}: the total rounds $T$ , constant $v>0$ and $ a \geq 0$, tuning parameter $\tau$, conjugate prior $\NN(\theta_0,v^{2}B_0^{-1})$ with $\theta_0=0$ and $B_0 =a I_d$, $f_0 = 0$
  \STATE \textbf{Initialization}: randomly choose $\alpha_t \in [K]$ for $t \in [\tau]$
   \STATE {\bf Update information}: $\FF_{\tau}$ according to \eqref{after_info}.
        \IF{$\II_{t}=\emptyset$}
        \STATE $\theta_1=\theta_0$, $B_1= B_0$
        \ELSE
       \STATE 
             $ B_1 = \sum_{(s,x_s,a_s,y_{s,a_s})\in \II_t } x_{s,a_s}x_{s,a_s}^T$, $f_1 = \sum_{(s,x_s,a_s,y_{s,a_s})\in \II_t } x_{s,a_s}y_{s,a_s}$
         and    $\theta_1 = B_1^{-1} f_1$
\ENDIF
 \FOR {$t = 1,2,\cdots,T-\tau$}
 \STATE {\bf Sample Model}: $\hat{\theta}_{t+\tau} \sim \NN(\theta_t,v^2B_t^{-1})$ 
 \STATE {\bf Select Action}: $\bar{a}_{t+\tau} \in \arg\max_{a \in [K]} \left\langle x_{t+\tau,a},\hat{\theta}_{t+\tau}\right\rangle$
 \STATE {\bf Update information}: $\FF_{t+\tau}$ according to \eqref{after_info}. Define $\II_{t+\tau}:=\FF_{t+\tau}-\FF_{t+\tau-1}$ as the new information at round $t+\tau$
        \IF{$\II_{t+\tau+1}=\emptyset$}
        \STATE $B_{t+1} = B_t$
        \STATE $\theta_{t+1} = \theta_{t}$
        \ELSE
       \STATE $B_{t+1}=B_t+\sum_{(s,x_s,a_s,y_{s,a_s})\in \II_{t+\tau} } x_{s,a_s}x_{s,a_s}^T$,\\ $f_{t+1}=f_t + \sum_{(s,x_s,a_s,y_{s,a_s})\in \II_{t+\tau} }x_{s,a_s}y_{s,a_s}$, and $\theta_{t+1} = B_{t+1}^{-1}f_{t+1}$
        \ENDIF
\ENDFOR
\end{algorithmic}
\end{algorithm}
Note that the update (line 17) is on the incremental form which is practically efficient.


\subsection{Regret Bounds}
Denote $\pi_{\tau}^{\text{PS}}$ as the posterior sampling policy described in Algorithm \ref{alg:PS-DGLM} with an exploration period  $\tau$. We have the following result.
%
\begin{theorem}\label{alg:BayesRegret}
Assume Assumptions \ref{a1}-\ref{delay}. There exists a universal constant $C:=C(C_1,C_2,M,\mu,\sigma_0,\sigma_G,\,\sigma,\kappa)>0$, such that if we run exploration with $\tau := C \left(d+\log(\frac{1}{\delta})\right)$,
\begin{eqnarray}\label{Rt_Bayes}
R^B_T(\pi_{\tau}^{\text{PS}}) 
= {O\left( d\log T \sqrt{T}+\sqrt{\sigma_G}\sqrt{Td}(\log(T))^{3/4}+(\sqrt{\mu+M}+\sigma_G)\sqrt{dT\log \left(T\right)}\right).}
\end{eqnarray}
\end{theorem}
For parameter definition, we refer to Table~\ref{tab:parameters} in Section \ref{app:table}. 

{
We follow the steps in \cite{RV2014} to prove the Bayesian regret bound in Theorem \ref{alg:BayesRegret}. The idea is the follows. We first decompose the Bayesian regret and the UCB the regret and build a connection between them. We then provide the Bayesian regret bound by utilizing a sequence of upper confidence bounds. We defer the details to Appendix \ref{proof}.
}

When $\{D_s\}_{s=1}^T$ satisfies some specific assumptions, we are able to provide tighter Bayesian regret bounds. 

\begin{corollary}\label{ts-cor}
Assume Assumptions \ref{a1}-\ref{delay}, we have the following result:
\begin{enumerate}
\item If there exists a constant $D_{\max}>0$ such that $\mathbb{P}(D_s \leq D_{\max})=1$ for all $s\in[T]$. Then, 
$$R^B_T(\pi_{\tau}^{\text{PS}}) = O\left( d\log T\sqrt{T}+{\sqrt{D_{\max}}}\sqrt{dT\log T}\right).$$
\item Assume $D_1,\cdots,D_T$ are \textbf{iid} non-negative random variables with mean $\mu_I$. Then,
$$R^B_T(\pi_{\tau}^{\text{PS}}) 
= {O\left( d\log T \sqrt{T}+\sqrt{\sigma_G}\sqrt{Td}(\log(T))^{3/4}+(\sqrt{\mu_I}+\sigma_G)\sqrt{dT\log \left(T\right)}\right).}
$$
\end{enumerate}
\end{corollary}
We defer the proof of Corollary \ref{ts-cor} to Appendix \ref{proof}. The results in Theorem \ref{alg:BayesRegret} and Corollary \ref{ts-cor} are comparable to the results in Section \ref{sec:UCB}.

%% file: extension.tex
\section{Extensions: History-dependent Delays}\label{sec:extension}
In previous sections, we have analyzed the regret bounds for both DUCB-GLCB and DTS-GLCB when delays are independent. In practice, such independence assumption may not hold and current delays may depend on historical delays. 
In this section, we explore two types of dependency structures for the delays. In section \ref{sec:markov_delay}, we discuss Markov delays where the stationary distribution is near-heavy-tail. In section \ref{sec:random_delay}, we discuss delays with random dependency structures but under a stronger assumption on the stationary distribution, which is lighter-than-sub-Gaussian.
\subsection{Markov Delays}\label{sec:markov_delay}

\begin{assumption}[Markov Delay]\label{markov_delay}
Let $\{D_t\}_{t=1}^T$ be a stationary Markov chain on the general state space $\mathcal{X}=\N^+$ with invariant distribution $\pi$. Given $D \sim \pi$ with $\mu_M = \E[D]$, we further assume that

\[
\P(D-\mu_M \geq x)  \leq \exp \left(\frac{-x^{1+q}}{2\sigma_M^2}\right),
\]
for some $q>0$ and $\sigma_M>0$.
\end{assumption}
Under Assumption \ref{markov_delay}, the stationary distribution $\pi$ can have near-heavy-tail property when $q$ is small.

Recall that $G_t = \sum_{s=1}^{t-1} \I\{s+ D_s \geq t\}$ is the number of missing reward and $G_t^* = \max_{1 \leq s \leq t}G_t$ is the running maximum number of missing reward. Under Assumption \ref{markov_delay}, $G_t$ and $G_t^*$ has the following properties and again this is the key to analyze regret bounds for both DUCB and DTS.

\begin{proposition}[Properties of $G_t$ and $G_t^\star$ under Markov delays]\label{prop:markov_delay}
Assume Assumption \ref{markov_delay} and $l_2$-spectral gap  $1-\lambda \in (0,1]$. Then,
\begin{enumerate}
\item For any $0<\delta<1$ and any $t$ we have, with probability at least $1-\delta$,
\begin{eqnarray}\label{Gt_markov}
G_t -\mu_M \leq A_2(\lambda)\log \left( \frac{1}{\delta}\right)+\sqrt{2A_1(\lambda)\mu_M \log \left( \frac{1}{\delta}\right)},
\end{eqnarray}
where $A_1(\lambda) = \frac{1+\lambda}{1-\lambda}$ and $A_2(\lambda) = \frac{1}{3}\I(\lambda=0)+\frac{5}{1-\lambda}\I(\lambda>0)$.
\item With probability at least $1-\delta$,
\begin{eqnarray}\label{Gt_max_markov}
G_T^* \leq \mu_M + A_2(\lambda) \log \left(\frac{T}{\delta}\right) +\sqrt{2A_1(\lambda) \mu_M \log \left(\frac{T}{\delta}\right)},
\end{eqnarray}
where $G_T^* =\max_{1\leq t \leq T} G_t$.
\item Define $W_t = \sum_{s \in T_t}X_s X_s^{\prime}$ where $X_t$ is drawn \textbf{iid}. from some distribution $\gamma$ with support in the unit ball $\mathbb{B}_d$. Furthermore, let $\Sigma := \mathbb{E}[X_t X_t^{\prime}]$ be the second moment matrix, and $B$ and $\delta >0$ be two positive constants. Then there exist positive, universal constants $C_1$ and $C_2$ such that $\lambda_{\min}(W_t)\geq B$ with probability at least $1-2\delta$, as long as 
\begin{eqnarray}\label{number_regularity_2}
t \geq \left( \frac{C_1\sqrt{d}+C_2\sqrt{\log(\frac{1}{\delta})}}{\lambda_{\min}(\Sigma)}\right)^2 +\frac{2B}{\lambda_{\min}(\Sigma)}+ \mu_M+ A_2(\lambda)\log \left(\frac{1}{\delta}\right)+\sqrt{2 A_1(\lambda)\mu_M\log\left(\frac{1}{\delta} \right)}.
\end{eqnarray}
\end{enumerate}
\end{proposition}
$\lambda$ is the $l_2$-spectral gap of the transition probability. {We refer the formal concepts and the definition of $l_2$-spectral gap to \cite[Section 2.2]{JSF2018}.} Proposition 3-1 is proved by utilizing the Berstein's inequality for general Markov chains (\cite[Theorem 1.1]{JSF2018}) and Proposition 3-2 is proved by applying union bound.

\begin{proof}{Proof of Proposition \ref{prop:markov_delay}.}
Recall  $G_t = \sum_{s=1}^{t-1} \I\{D_s \geq t-s\}$. Define $f_i(D_i)=\I\{D_s \geq t-s\}-p_i$ with $p_i = \P(D_s \geq t-s)$. Then $\E[ f_i(D_i)]=0$, $\V [f_i(D_i)]=p_i(1-p_i)\leq p_i$, and $\sum_{s=1}^{t-1}\V [f_i(D_i)]\leq\sum_{s=1}^{t-1}p_s<\mu_M$.

From \cite[Theorem 1.1]{JSF2018}, we have
\begin{eqnarray}\label{markov_temp}
\P\left(\sum_{s=1}^{t-1}f_i(D_i)>x\right) \leq \exp \left( -\frac{x^2}{2(A_1(\lambda)\mu_M+A_2(\lambda) x)}\right).
\end{eqnarray}
Note that the right hand side in \eqref{markov_temp} is independent of $t$. Technically speaking, this is because the summation of the variance $\sum_{s=1}^{t-1}\V [f_i(D_i)]$ is upper bounded by $\mu_D$ which is independent of $t$. Therefore, Property 1 in Proposition \ref{prop:markov_delay} holds for any $t \geq 1$.

Property 2 holds by the union bound and Property 1,
\begin{eqnarray*}
\P\left(\max_{1\leq t \leq T}G_t>\mu_M + A_2(\lambda) \log \left(\frac{T}{\delta}\right) +\sqrt{2A_1(\lambda) \mu_M \log \left(\frac{T}{\delta}\right)}\right) \\
\leq \sum_{t=1}^T \P\left(G_t>\mu_M + A_2(\lambda) \log \left(\frac{T}{\delta}\right) +\sqrt{2A_1(\lambda) \mu_M \log \left(\frac{T}{\delta}\right)}\right) \leq T\times \frac{\delta}{T}.
\end{eqnarray*}
Therefore, the following holds with probability no smaller than $1-\delta$,
\[
\P\left(\max_{1\leq t \leq T}G_t<\mu_M + A_2(\lambda) \log \left(\frac{T}{\delta}\right) +\sqrt{2A_1(\lambda) \mu_M \log \left(\frac{T}{\delta}\right)}\right) .
\]
\end{proof}

\remark{In Assumption \ref{markov_delay}, we assume the Markov chain $\{D_t\}_{t=1}^{\infty}$ starts from the stationary distribution. In fact, our analysis works with any initial distribution by further assuming a mild uniform mixing condition.}

Now we are ready to state the main results for DUCB and DTS under Markov delays.

\begin{theorem}[DUCB bound with Markov delays]\label{thm: UCBregret_markov}
Assume Assumptions \ref{a1} and \ref{markov_delay}.  Fix any $\delta$. There exists a universal constant $C:=C(C_1,C_2,\mu_M,\sigma_0,\hat{\sigma},\,\sigma,\kappa,\lambda)>0$, such that if we run DUCB-GLCB with $\tau := C \left(d+\log(\frac{1}{\delta})\right)$ and $\beta_t = \frac{\hat{\sigma}}{\kappa} \sqrt{\frac{d}{2} \log \left(1+\frac{2(t-G_t)}{d}\right)+\log(\frac{1}{\delta})}+{\sqrt{G_t}} $, then, with probability at least $1-5\delta$, the regret of the algorithm is upper bounded by
\begin{eqnarray}\label{Rt_markov}
R_T &\leq& \tau+ L_{g} \left[ 2{\sqrt{\mu_M}} \sqrt{2T d \log \left(\frac{T}{d}\right)}  
+2{\sqrt{A_2(\lambda)\log \left(\frac{T}{\delta}\right)}}\sqrt{2T d \log \left(\frac{T}{d}\right)}\right.\nonumber\\
&&\left.+2{\left({2A_1(\lambda) \mu_M \log \left(\frac{T}{\delta}\right)}\right)^{1/4}}\sqrt{2d T \log \left(\frac{T}{d}\right)}
+\frac{2d\hat{\sigma}}{\kappa} \log \left( \frac{T}{d \delta}\right)\sqrt{T}\right].\end{eqnarray}
 \end{theorem}
Therefore, we have $R_T = {{O}\left(\mu_M\sqrt{ dT}\log^{\frac{3}{4}}(T)+ d \log (T)\sqrt{T}+\sqrt{\mu_M}\sqrt{d}\sqrt{T\log(T)}+\sqrt{T d}\log(T)\right)}$.

With the result in Proposition ~\ref{prop:markov_delay}, we can show Theorem \ref{thm: UCBregret_markov} by adopting similar ideas as in Theorem \ref{thm: UCBregret}.

\begin{theorem}[DTS bound with Markov delays]\label{thm:BayesRegret_markov}
Assume Assumptions \ref{a1} and \ref{markov_delay}. There exists a universal constant $C:=C(C_1,C_2,\sigma_0,\sigma,\hat{\sigma},\kappa)>0$, such that if we run exploration with $\tau := C \left(d+\log(\frac{1}{\delta})\right)$,
$$R^B_T(\pi_{\tau}^{\text{PS}}) 
= {{O}\left(\mu_M\sqrt{ dT}\log^{\frac{3}{4}}(T)+ d \log (T)\sqrt{T}+\sqrt{\mu_M}\sqrt{d}\sqrt{T\log(T)}+\sqrt{T d}\log(T)\right)}.
$$
\end{theorem}


%
%

\subsection{Delays with Random Dependency Structure}\label{sec:random_delay}

In this section, we assume the following assumption on the delay sequence  $\{D_s\}_{s=1}^{\infty}$.
\begin{assumption}\label{random_delay}
Assume $\{D_s\}_{s=1}^{\infty}$ has a stationary distribution $\pi$ and $D\sim\pi$ satisfies
\[
\P(D-\mu_R \geq x) \leq \exp \left( -\frac{x^{2(1+q)}}{\sigma_R^2}\right),
\]
for some $\sigma_R>0$ and $q>0$. Here $\E[D]=\mu_R$.
\end{assumption}
Note that Assumption \ref{random_delay} only assumes the tail probability of the stationary distribution without any restriction on the dependency structure among $\{D_s\}_{s=1}^{\infty}$. Under Assumption \ref{random_delay}, $D^{1+q}$ is sub-Gaussian, which is stronger than the assumption on the envelope distribution described in Assumption \ref{delay}.
\begin{proposition}[Properties of $G_t$ and $G_t^\star$ under delays with random structure]\label{prop:random}
Assume Assumption \ref{random_delay}. Denote $\sigma_G = \frac{\sigma_R}{c}$ with $c = \frac{1}{\sum_{i=1}^{\infty}\frac{1}{i^{1+q}}}>0$. Then,
\begin{enumerate}
\item $G_t$ is sub-Gaussian. Moreover, for all $t\geq 1$,
\begin{eqnarray}\label{Gt_random}
\P \left(G_t \geq \mu_R+x \right) \leq C_4 \exp\left(\frac{-x^2}{2\sigma^2_G}\right),
\end{eqnarray}
with $0<C_4\leq 2\sigma_{R}^2+1$.
\item With probability $1-\delta$,
\begin{eqnarray}\label{Gt_max_random}
G_T^* \leq \mu_R +\sigma_G \sqrt{2 \log(T)}+\sigma_G \sqrt{2 \log\left(\frac{C_4}{\delta} \right)},
\end{eqnarray}
where $G_T^* =\max_{1\leq s \leq T} G_s$.
\item Define $W_t = \sum_{s \in T_t}X_s X_s^{\prime}$ where $X_t$ is drawn \textbf{iid}. from some distribution $\gamma$ with support in the unit ball $\mathbb{B}_d$. Furthermore, let $\Sigma := \mathbb{E}[X_t X_t^{\prime}]$ be the second moment matrix, and $B$ and $\delta >0$ be two positive constants. Then there exist positive, universal constants $C_1$ and $C_2$ such that $\lambda_{\min}(W_t)\geq B$ with probability at least $1-2\delta$, as long as 
\begin{eqnarray}\label{number_regularity_random}
t \geq \left( \frac{C_1\sqrt{d}+C_2\sqrt{\log(\frac{1}{\delta})}}{\lambda_{\min}(\Sigma)}\right)^2 +\frac{2B}{\lambda_{\min}(\Sigma)}+ \mu_R+\sigma_G \sqrt{2 \log\left(\frac{C_4}{\delta} \right)}.
\end{eqnarray}
\end{enumerate}
\end{proposition}

Proposition \ref{prop:random} is proved by utilizing the finiteness of $\sum_{i=1}^{\infty}\frac{1}{i^{1+q}}$ (for any $q>0$) and by constructing a union bound with proper event decompositions. We defer the details to Appendix \ref{proof}.

Given Proposition \ref{random_delay}, we can show the following results on DUCB and DTS with random dependency structures for the delay sequence.

\begin{theorem}[DUCB bound under delays with random structure]\label{thm: UCBregret_random}
Assume Assumptions \ref{a1} and \ref{random_delay}.  Fix any $\delta$. There exists a universal constant $C:=C(C_1,C_2,\mu_R,\sigma_0,\hat{\sigma},\,\sigma_R,\kappa)>0$, such that if we run DUCB-GLCB with $\tau := C \left(d+\log(\frac{1}{\delta})\right)$ and $\beta_t = \frac{\hat{\sigma}}{\kappa} \sqrt{\frac{d}{2} \log \left(1+\frac{2(t-G_t)}{d}\right)+\log(\frac{1}{\delta})}+{\sqrt{G_t}} $, then, with probability at least $1-5\delta$, the regret of the algorithm is upper bounded by
\begin{eqnarray}\label{Rt_random}
R_T &\leq& \tau+ L_{g} \left[2{\sqrt{\mu_R}} \sqrt{2T d \log \left(\frac{T}{d}\right)} + 2{\sqrt{\sigma_G}({2 \log T})^{1/4}} \sqrt{ 2Td\log \left(\frac{T}{d}\right)} \right.\nonumber \\
&& \,\,+ \left.2{\sqrt{\sigma_G} \left({2\log \left(\frac{C_3}{\delta}\right)}\right)^{1/4}} \sqrt{2T d\log \left(\frac{T}{d}\right)}+\frac{2d\hat{\sigma}}{\kappa} \log \left( \frac{T}{d \delta}\right)\sqrt{T}\right],\end{eqnarray}
where $\sigma_G=\frac{\sigma_R}{c}$.
 \end{theorem}

\begin{theorem}[DTS bound under delays with random structure]\label{alg:BayesRegret_random}
Assume Assumptions \ref{a1} and \ref{random_delay}. There exists a universal constant $C:=C(C_1,C_2,\mu_R,\sigma_0,\hat{\sigma},\,\sigma_R,\kappa)>0$, such that if we run exploration with $\tau := C \left(d+\log(\frac{1}{\delta})\right)$,
$$R^B_T(\pi_{\tau}^{\text{PS}}) 
= O\left(d\log T\sqrt{T}+{\sqrt{\mu_R}}\sqrt{dT\log T}+{\sqrt{\sigma_G}\left(\log T\right)^{3/4}\sqrt{Td}}\right).
$$
\end{theorem}

By assuming lighter-than-sub-Gaussian tails on the stationary distribution, one can allow very general structure or even no structure on the delay sequence. The regret bound is on the order of $O\left(d\log T\sqrt{T}+{\sqrt{\mu_R}}\sqrt{dT\log T}+{\sqrt{\sigma_G}\left(\log T\right)^{3/4}\sqrt{Td}}\right)$, which is consistent with the results in Sections \ref{sec:UCB} and \ref{sec:ts}.

%% file: appendix.tex
\section{Table of Parameters}\label{app:table}

\begin{table}[H]
  \begin{center}
    \begin{tabular}{|r|l|} 
      \hline
      \textbf{Notation} & \textbf{Definition} \\
      \hline
      $K$ & number of arms \\
        $d$ & feature dimension \\
        $\kappa$& $\inf_{\{\|x\|\leq 1, \|\theta-\theta^*\|\leq 1\}}\dot{g}(x^{\prime}\theta)$\\
        $\theta^*$ & unknown parameter in GLCB model\\
        $\hat{\sigma}$& sub-Gaussian parameter for noise $\epsilon_t$\\
         $L_g$& upper bound on $\dot{g}$\\
        $M_{g}$& upper bound on $\ddot{g}$\\
        $\sigma^2_0$&lower bound on $\lambda_{\min} (\E[\frac{1}{K} \sum_{a \in [K]}x_{t,a}x_{t,a}^{\prime}])$\\
        $\xi$ & tail-envelope distribution for the delays\\
        $q$& parameter to characterize the tail-envelope distribution $\xi$ \\
        $\mu$& expectation of the tail-envelope distribution $\xi$\\
        $M$& parameter of $\xi$\\
        $\sigma$& parameter of $\xi$\\
        $\mu_I$& expectation of \textbf{iid} delays\\
        $\sigma_I$& parameter for \textbf{iid} delays\\
         $\mu_M$& expectation of Markov delays\\
        $\sigma_M$& parameter for Markov delays\\
         $\mu_R$& expectation of random structured delays\\
        $\sigma_R$& parameter for  random structured delays\\
         $\sigma_G$& sub-Gaussian parameter of $G_t$\\
        $D_{max}$& upper bound on bounded delays\\
        \hline
    \end{tabular}
   
     \caption{Parameters in the GLCB model with delays.} \label{tab:parameters}
  \end{center}
\end{table}

\section{Auxiliary Results}

\begin{theorem}[Maximum over a finite set, \cite{wainwright2019}]\label{thm7}
 Let $X_1,\cdots,X_n$ be centered $\sigma$-sub-Gaussian random variables. (i.e. $\E [\exp({\lambda X_i})]\leq \exp\left(\frac{\lambda^2\sigma^2}{2}\right)$). Then,
 \[
 \E \left(\max_{1\leq i\leq n} X_i\right) \leq \sigma \sqrt{2 \log(n)},
 \]
 and
  \[
 \E \left(\max_{1\leq i\leq n} |X_i|\right) \leq \sigma \sqrt{2 \log(2n)}.
 \]
 Moreover, for any $t \ge 0$,
 \[
\P(\max_{1\leq i\leq n} X_i > t) \leq\exp \left(-\frac{t^2}{2\sigma^2}+\log n \right),
\]
and 
\[
\P(\max_{1\leq i\leq n}| X_i| > t) \leq2\exp \left(-\frac{t^2}{2\sigma^2}+\log n\right).
\]
\end{theorem}
Note that the random variables in Theorem \ref{thm7} need not be independent.

\begin{theorem}[ Sub-Gaussian parameter for indicators, \cite{OS2014} ]\label{thm8}
Let $p \in [0,1]$ and let $\eta$ be a centered random variable such that $\P(\eta=1-p)=p$ and $\P(\eta=-p)=1-p$, then
\[
\E[\exp(\lambda \eta)] \leq \exp (\lambda^2 Q(p)),
\]
where $Q(p)=\frac{1-2p}{4\log(\frac{1-p}{p})}$.
\end{theorem}

\begin{theorem}[Hoeffding Bound, \cite{wainwright2019}]\label{thm9}
Let $X_1,\cdots,X_n$ be independent random variables. Assume $X_i$ has mean $\mu_i$ and sub-Gaussian parameter $\sigma_i$. Then for all $t \geq0$, we have
\[
\P \left(\sum_{i=1}^n (X_i-\mu_i)\geq t \right) \leq \exp \left(
-\frac{t^2}{2 \sum_{i=1}^n \sigma_i^2}\right).
\]
\end{theorem}

\section{Further discussion on $G_t$}\label{further_G}
\begin{proposition}[Properties of $G_t$ and $G_t^\star$]\label{prop:delay_second_choice}
Assume Assumption \ref{delay} with $q>0$. Denote $\sigma_G =  \sqrt{\frac{I}{4}+\frac{\sigma^2 (1+q)}{q}}$ with $I =\max\left\{\sqrt[1+q]{2\log( 2) \sigma^2},\sqrt[q]{\frac{2\sigma^2}{1+q}}+1\right\} $. Then,
\begin{enumerate}
\item $G_t$ is sub-Gaussian. Moreover, for all $t\geq 1$,
\begin{eqnarray}\label{Gt_second}
\P \left(G_t \geq 2(\mu + M) +x \right) \leq \exp\left(\frac{-x^2}{2\sigma^2_G}\right).
\end{eqnarray}
\item With probability $1-\delta$,
\begin{eqnarray}\label{Gt_max_second}
G_T^* \leq 2(\mu+M)+\sigma_G \sqrt{2 \log(T)}+\sigma_G \sqrt{2 \log\left(\frac{1}{\delta} \right)},
\end{eqnarray}
where $G_T^* =\max_{1\leq s \leq T} G_s$.
\item Define $W_t = \sum_{s \in T_t}X_s X_s^{\prime}$ where $X_t$ is drawn \textbf{iid}. from some distribution $\gamma$ with support in the unit ball $\mathbb{B}_d$. Furthermore, let $\Sigma := \mathbb{E}[X_t X_t^{\prime}]$ be the second moment matrix, and $B$ and $\delta >0$ be two positive constants. Then there exist positive, universal constants $C_1$ and $C_2$ such that $\lambda_{\min}(W_t)\geq B$ with probability at least $1-2\delta$, as long as 
\begin{eqnarray}\label{number_regularity_second}
t \geq \left( \frac{C_1\sqrt{d}+C_2\sqrt{\log(\frac{1}{\delta})}}{\lambda_{\min}(\Sigma)}\right)^2 +\frac{2B}{\lambda_{\min}(\Sigma)}+ 2(\mu+M)+\sigma_G \sqrt{2 \log\left(\frac{1}{\delta} \right)}.
\end{eqnarray}
\end{enumerate}
\end{proposition}

The key idea of the proof is to utilize the smallest sub-Gaussian parameter for indicator functions. The details is deferred to Appendix \ref{proof}. 
In most of the stochastic contextual bandit literature, with or without delays, the most popular approach is to apply Bernstein' inequality or Hoeffding bound. In Proposition \ref{prop:delay}, we show that an essential stochastic analysis approach with stopping times can sharpen the result compared with Hoeffding bound. That is, the statement in Proposition \ref{prop:delay_second_choice} is weaker than that in Proposition \ref{prop:delay}.

%
%

%
%

\section{Missing Proofs}\label{proof}
In this section, we provide the proofs for Propostion \ref{prop:delay}, Theorem \ref{thm: UCBregret}, Proposition \ref{prop:Gt}, Theorem \ref{alg:BayesRegret}, Corollary \ref{ts-cor},  Proposition \ref{prop:random} and Proposition \ref{prop:delay_second_choice}.

\begin{proof}[Proof of Proposition \ref{prop:delay}.] Now let us prove all three properties in Proposition \ref{prop:delay}.\\
\textbf{Property 1.}  Let $\tilde{D}_{k_i}$ be a random variable such that $\tilde{D}_{k_i}\geq -(\mu+M)$ almost surely, $\E[\tilde{D}_{k_i}]\leq 0$ and $\P(\tilde{D}_{k_i}\geq x) \leq \exp{\left(-\frac{x^{1+q}}{2\sigma^2}\right)}$ for $x \geq 0$. One can view $\tilde{D}_{k_i}$ as a shifted delay. Define $\tilde{I}_i = \I \left(\tilde{D}_{k_i} \geq i\right)-p_i$ with $p_i = \P(\tilde{D}_{k_i} \geq i)$. Then $\P \left(\tilde{I}_i=1-p_i\right)=p_i$ and $\P(\tilde{I}_i=p_i)=1-p_i$.

Similar to the highlighted proof with {\bf{iid}} delays and $q=0$ in Section \ref{sec:pre},
let $\tilde{G}=\sum_{i=1}^{\infty}\tilde{I}_i$ and define the following sequence of stopping times, $(k \geq 1)$,
\[
T(k) = \inf \{t >T(k-1):D_t \geq t\},
\]
where $T(k)$ is the time of the $k^{\text th}$ success. Therefore,
\begin{eqnarray*}
\P(\tilde{G} \geq j) &=& \P(T(1)<\infty,T(2)<\infty,\cdots,T(j-1)<\infty,T(j)<\infty)\nonumber\\
&=& \Pi_{k=1}^{j}\P\left(T(k)<\infty \vert T(i)<\infty \,\,\text{ for }\,\, i \leq k-1\right)\\
&=& \Pi_{k=2}^{j}\P\left(T(k)<\infty \vert T(k-1)<\infty \right)\P\left(T(1)<\infty \right) \label{equal}\\
&\leq& \Pi_{k=1}^{j} \left(\sum_{i=k}^{\infty}\exp \left(-\frac{i^{1+q}}{2\sigma^2}\right)\right)\\
&\leq&\Pi_{k=1}^{j}  \left((2\sigma^{2}+1)\exp \left(-\frac{k^{1+q}}{2\sigma^2}\right)\right)\\
&\leq& (2\sigma^2+1)^{j}\exp\left(-\frac{(j-1)^{2+q}}{2(2+q)\sigma^2}\right)\\
 &\leq& (2\sigma^2+1)^{j}\exp\left(-\frac{(j-1)^{2}}{2(2+q)\sigma^2}\right)
\end{eqnarray*}

Therefore, $\tilde{G}$ is sub-Gaussian  with parameter $\sigma_G := \sigma\sqrt{2+q}$. With probability $1-\delta$, we have, 
\[
\tilde{G} \leq  \sigma\sqrt{2(2+q)}\sqrt{\log\left(\frac{1}{\delta}\right)}+2(2+q)\sigma^2\log(2\sigma^2+1)+1.
\]

Define $\tilde{G}_t = \sum_{i=1}^t \tilde{I}_i$. Similarly, for any $t \ge 1$, 
\begin{eqnarray*}
\tilde{G}_t &\leq&  \sigma\sqrt{2(2+q)}\sqrt{\log\left(\frac{1}{\delta}\right)}+2(2+q)\sigma^2\log(2\sigma^2+1)+1,\\
&=&\sigma_G\sqrt{\log\left(\frac{1}{\delta}\right)}+\sigma_G^2\log(C_3)+1
\end{eqnarray*}
holds with probability $1-\delta$, where $C_3 = 2\sigma^2+1$.

Recall $G_t = \sum_{s=1}^{t-1}\I(D_s\geq t-s)$. When $t\leq \mu+M-1$, ${G}_{t}\leq \mu+M$. When $t\geq \mu+M-1$, specifying $k_i=t-(\mu+M)-i$ and $\tilde{D}_{k_i} = {D}_{i}-\mu-M$, 
\begin{eqnarray*}
G_t &=& \sum_{s=1}^{t-1}\I(D_s\geq t-s) \nonumber\\
&=& \sum_{s=1}^{t-\mu-M-1}\I(D_s\geq t-s)+\sum_{s=t-\mu-M}^{t-1}\I(D_s\geq t-s)\nonumber\\
&=& \sum_{s=1}^{t-\mu-M-1}\I(D_s-\mu-M\geq t-s-\mu-M)+\sum_{s=t-\mu-M}^{t-1}\I(D_s\geq t-s)\nonumber\\
&\leq& \sum_{s=1}^{t-\mu-M-1}\I(D_s-\mu-M\geq t-s-\mu-M)+\mu+M\nonumber\\
&=& \sum_{i=1}^{t-\mu-M-1}\I(D_{t-(\mu+M)-i}-\mu-M\geq i)+\mu+M\,\,\,\,\,\,\,(i=t-s-\mu-M)\\
&=& \sum_{i=1}^{t-\mu-M-1}\I(\tilde{D}_{k_i}\geq i)+\mu+M .
\end{eqnarray*}

Hence,
\begin{eqnarray}
G_t&\leq&\sum_{i=1}^{t-\mu-M-1}[\I(\tilde{D}_{k_i}\geq i)-p_i]+(\sum_{i=1}^{t-\mu-M-1} p_i)+\mu+M \nonumber\\
&=&\sum_{i=1}^{t-\mu-M-1}\tilde{I}_i+(\sum_{i=1}^{t-\mu-M-1} p_i)+\mu+M\nonumber\\
&\leq&\sum_{i=1}^{t-\mu-M-1}\tilde{I}_i+(\mu+M)+\mu+M\nonumber\\
&=&\tilde{G}_{t-\mu-M-1}+2(\mu+M) .\label{ggt}
\end{eqnarray}

Therefore, we arrive at $\tilde{G} \leq \tilde{G}_{t-\mu-M-1}+2(\mu+M)$ with specific choice of $k_i=t-(\mu+M)-i$ and $\tilde{D}_{k_i} = {D}_{i}-\mu-M$.

Finally, with probability $1-\delta$,
\[
{G}_t \leq 2(\mu+M) +\sigma_G\sqrt{2\log\left(\frac{1}{\delta}\right)}+2\sigma_G^2\log C_3+1,
\]
where $\sigma_G =  \sigma \sqrt{2+q} $ and $C_3 = 2\sigma^2+1$.

\textbf{Property 2.} Further define  $\tilde{G}_T^* = \max_{1 \leq t\leq T} \{\tilde{G}_t\}$ as the running maximum of correlated sub-exponentials $\tilde{G}_t$ up to time $T$. By the union bound,
\begin{eqnarray*}
\P \left(\tilde{G}_T^* \geq \sigma_G \sqrt{2\log T} + x \right) 
&\leq& \sum_{t=1}^T\P(\tilde{G}_t \geq \sigma_G \sqrt{2\log T} + x) \\
&\leq& T (2\sigma^2+1)^{ \sigma_G \sqrt{2\log T} + x}\, \exp\left( -\frac{(\sigma_G \sqrt{2\log T}+x-1)^2}{2\sigma_G^2}\right)\\
&=& T (2\sigma^2+1)^{ \sigma_G \sqrt{2\log T} + x}\exp\left(-\frac{(x-1)^2}{2 \sigma_G^2} -\frac{2(x-1)\sigma_G \sqrt{2\log T}}{2 \sigma_G^2}-\log T\right)\\
&=&  (2\sigma^2+1)^{ \sigma_G \sqrt{2\log T} + x} \exp\left(-\frac{(x-1)^2}{2 \sigma_G^2} -\frac{2(x-1)\sigma_G \sqrt{2\log T}}{2 \sigma_G^2}\right)\\
&\leq&  (2\sigma^2+1)^{ \sigma_G \sqrt{2\log T} + x}\exp \left(-\frac{(x-1)^2}{2 \sigma_G^2}\right).
\end{eqnarray*}
Therefore, with probability $1-\delta$,
\[
\tilde{G}_T^*\leq \sigma_G \sqrt{2\log T}+ 2\sigma_G^2\log C_3 +\sqrt{2\sigma_G^2}\sqrt{\log\left( \frac{1}{\delta}\right)+\log C_3 \sigma_G \sqrt{2\log T}+\log C_3}+1,
\]
where $ C_3= 2\sigma^2+1$.

Recall that ${G}_{T}^*=\max_{1\leq t\leq T} G_t$. 
When $T\leq \mu+M-1$, ${G}_{T}^*\leq \mu+M$. When $T\geq \mu+M-1$, specifying $k_i=T-(\mu+M)-i$ and $\tilde{D}_{k_i} = {D}_{k_i}-\mu-M$, we have
 $${G}_{T}^* \leq \tilde{G}_T^*+2(\mu+M).$$  
The derivation is  similar to the analysis in \eqref{ggt}. 
 
Therefore, with probability $1-\delta$, we have
\[
{G}_T^*\leq 2( \mu+M)+ \sigma_G \sqrt{2\log T}+ 2\sigma_G^2\log C_3 +\sqrt{2\sigma_G^2}\sqrt{\log\left( \frac{1}{\delta}\right)+\log C_3 \sigma_G \sqrt{2\log T}+\log C_3}+1.
\]

Above result implies that ${G}_T^*=O(\sigma_G\sqrt{\log T})$.

\textbf{Property 3.}
Given a fixed sequence $\{G_s\}_{s=1}^{\infty}$, from \cite{VERSHYNIN2010} and \cite{LLZ2017}, $\lambda_{\min}(W_t)\geq B$ with probability $1-\delta$, when
\begin{eqnarray}
t \geq \left( \frac{C_1\sqrt{d}+C_2\sqrt{\log(\frac{1}{\delta})}}{\lambda_{\min}(\Sigma)}\right)^2 +\frac{2B}{\lambda_{\min}(\Sigma)}+ G_t.
\end{eqnarray}
Combining above with \eqref{Gt}, we have the desired result.
\end{proof}

\begin{proof}[Proof of Theorem \ref{thm: UCBregret}.]
We first bound the one-step regret. To do so, fix $t$ and let $X_t^*=x_{t,a^*_t}$ and $\Delta_t = \hat{\theta}_t-\theta^*$, where $a_t^* = \arg \max_{a \in [K]} \mu(x_{t,a}^{\prime}\theta^*)$ is an optimal action at round $t$. The selection of $a_t$ in DUCB-GLCB implies
\[
\langle X_t^*,\hat{\theta}_t\rangle+\beta_t \|X_t^*\|_{V_t^{-1}}\leq \langle X_t,\hat{\theta}_t\rangle + \beta_t \|X_t\|_{V_t^{-1}}.
\]
Then we have
\begin{eqnarray}
\langle X_t^*,\theta^*\rangle-\langle X_t,\theta^*\rangle &=& \langle X_t^*-X_t,\hat{\theta}_t\rangle - \langle X_t^*-X_t,\hat{\theta}_t-\theta^*\rangle\\
&\leq& \beta_t (\|X_t\|_{V_t^{-1}}-\|X^*_t\|_{V_t^{-1}}) + \|X_t^*-X_t\|_{V_t^{-1}} \|\Delta\|_{V_t}.
\end{eqnarray}
Therefore, to bound $\langle X_t^*,\theta^*\rangle-\langle X_t,\theta^*\rangle$, it suffices to bound $\|\Delta\|_{V_t}$ and $\|X_t\|_{V_t^{-1}}$.


Suppose {\color{black}$\lambda_{\min}(W_{\tau+1})\geq 1$}, for any $\delta \in [\frac{1}{T},1)$ define event
\[
\EE_{\Delta} := \left\{ \|\Delta\|_{W_t} \leq \frac{\hat{\sigma}}{\kappa} \sqrt{\frac{d}{2} \log \left(1+\frac{2(t-G_t)}{d}\right)+\log \left(\frac{1}{\delta}\right)}\right\}.
\]
From Lemma 2 in (\cite{LLZ2017}), then event $\EE_{\Delta}$ holds for all $t \geq \tau$ with probability at least $1-\delta$.

\begin{eqnarray*}
\|\Delta_t\|^2_{V_t} &=& \Delta_t^{\prime}V_t\Delta_t =  \Delta_t^{\prime}\left(W_t+\sum_{s\in M_t}X_s X_s^{\prime}\right)\Delta_t\\
&=&\Delta_t^{\prime}W_t\Delta_t + \sum_{s\in M_t} \Delta_t^{\prime}X_s X_s^{\prime}\Delta_t\\
&\leq&\Delta_t^{\prime}W_t\Delta_t + \sum_{s\in M_t} \|\Delta_s\|^2\|X_s\|^2  \\
&\leq& \|\Delta_t\|^2_{W_t} +G_t\|\Delta_t\|^2.
\end{eqnarray*}
{When $\lambda_{\min}(W_{t}) \geq 16 \sigma^2 \frac{d+\log(\frac{1}{\delta})}{\kappa^2}$}, from Lemma 7 in (\cite{LLZ2017}), with probability $1-\delta$,
$$\|\Delta_t\|^2 \leq \frac{4\sigma}{\kappa}\sqrt{\frac{d+\log(\frac{1}{\delta})}{\lambda_{\min}(W_t)}}\leq 1.$$
Therefore, {when $\lambda_{\min}(W_{t}) \geq 16 \sigma^2 \frac{d+\log(\frac{1}{\delta})}{\kappa^2}$}, with probability $1-2\delta$,
\begin{eqnarray}\label{delta_v_bound}
\|\Delta_t\|_{V_t} &\leq&  \frac{\hat{\sigma}}{\kappa} \sqrt{\frac{d}{2} \log \left(1+\frac{2(t-G_t)}{d}\right)+\log \left(\frac{1}{\delta}\right)+G_t}\nonumber\\
&\leq&  \frac{\hat{\sigma}}{\kappa} \sqrt{\frac{d}{2} \log \left(1+\frac{2(t-G_t)}{d}\right)+\log \left(\frac{1}{\delta}\right)}+\sqrt{G_t}.
\end{eqnarray}
Let us come back to the satisfaction of conditions $\lambda_{\min}(W_{t}) \geq 16 \sigma^2 \frac{d+\log(\frac{1}{\delta})}{\kappa^2}$ and $\lambda_{\min}(W_{\tau+1})\geq 1$. From Proposition \ref{prop:delay}, $\lambda_{\min}(W_t)\geq \max\left\{1,16 \sigma^2 \frac{d+\log(\frac{1}{\delta})}{\kappa^2}\right\}$ with probability $1-2\delta$, when 
\begin{eqnarray}\label{tau}
t \geq&& \left( \frac{C_1\sqrt{d}+C_2\sqrt{\log(\frac{1}{\delta})}}{\lambda_{\min}(\Sigma)}\right)^2 +\frac{2\max\{1,16 \sigma^2 \frac{d+\log(\frac{1}{\delta})}{\kappa^2}\}}{\lambda_{\min}(\Sigma)} \\
&+& 2(\mu+M) +\sigma_G\sqrt{2\log\left(\frac{1}{\delta}\right)}+2\sigma_G^2\log C_3+1:=\tau.
\end{eqnarray}
We now choose $\beta_t = \frac{\hat{\sigma}}{\kappa} \sqrt{\frac{d}{2} \log \left(1+\frac{2(t-G_t)}{d}\right)+\log(\frac{1}{\delta})}+\sqrt{G_t}$. If $\EE_t$ holds for all $t \geq \tau$, then,
\begin{eqnarray}\label{est1}
\langle X_t^*,\theta^*\rangle-\langle X_t,\theta^*\rangle
\leq \beta_t \left(\|X_t\|_{V_t^{-1}}-\|X_t^*\|_{V_t^{-1}}+\|X_t^*-X_t\|_{V_t^{-1}}\right).
\end{eqnarray}
Suppose there is an integer $m$ such that $\lambda_{\min}(V_{m+1})\geq 1$, from Lemma 2 in \cite{LLZ2017}, we have
\begin{eqnarray}\label{lemma1}
\sum_{t=m+1}^{m+n} \|X_t\|_{V_t^{-1}} \leq \sqrt{2dn \log \left(\frac{n+m}{d}\right)}.
\end{eqnarray}
for  all $n \geq 0$.
Combine \eqref{est1} and \eqref{lemma1}, we have
\begin{eqnarray*}
\sum_{t=\tau+1}^T(\langle X_t^*,\theta^*\rangle-\langle X_t,\theta^*\rangle) &\leq& 2 \max_{1 \leq t\leq T}\{\beta_t\} \sqrt{2T d \log \left(\frac{T}{d}\right)}\\
&\leq& 2\left[\frac{\hat{\sigma}}{\kappa} \sqrt{\frac{d}{2} \log \left(1+\frac{2T}{d}\right)+\log \left(\frac{1}{\delta}\right)}+\sqrt{G^*_T} \right]\sqrt{2T d \log \left(\frac{T}{d}\right)}\\
&\leq& 2\sqrt{G_T^*} \sqrt{2T d \log \left(\frac{T}{d}\right)} + \frac{2d \hat{\sigma}}{\kappa} \log \left(\frac{T}{d\delta}\right)\sqrt{T}.
\end{eqnarray*}
Note that $g$ is an increasing Lipschitz function with Lipschitz constant $L_{g}$ and the $g$ function is bounded between 0 and 1. The regret of algorithm DUCB-GLCB can be upper bounded as

\begin{eqnarray}\label{regret_inequaltiy}
R_T &\leq& \tau + L_{g} \sum_{t=\tau+1}^T (\langle X_t^*,\theta^*\rangle-\langle X_t,\theta^*\rangle)\nonumber\\
 &\leq& \tau + L_{g} \left( 2\sqrt{G_T^*} \sqrt{2T d \log \left(\frac{T}{d}\right)} + \frac{2d \hat{\sigma}}{\kappa} \log\left(\frac{T}{d\delta}\right)\sqrt{T}\right).
\end{eqnarray}
Combining with the results in \eqref{Gt_max}, \eqref{delta_v_bound} and \eqref{tau}, with probability $1-5\delta$,
\begin{eqnarray*}
R_T &\leq& \tau+ L_{g} \left[ {\sqrt{2(\mu+M)}} 2\sqrt{2T d \log \left(\frac{T}{d}\right)} + {\sqrt{\sigma_G\sqrt{2\log T }}} 2\sqrt{2d \log \left(\frac{T}{d}\right)T}+\frac{2d \hat{\sigma}}{\kappa} \log \left( \frac{T}{d \delta}\right)\sqrt{T} \right. \\
&& \,\,+\left.2\sqrt{2T d \log \left(\frac{T}{d}\right)}{\left(\sqrt{\sigma_G } \left({2\log\left( \frac{1}{\delta}\right)+2\log C_3 \sigma_G \sqrt{2\log T}+2\log C_3}\right)^{1/4}+\sqrt{1+2\sigma_G^2\log C_3}\right)}\right] \\
&=&\tau+ L_{g} \left[ {4\sqrt{(\mu+M)}} \sqrt{T d \log \left(\frac{T}{d}\right)} + {2^{7/4}\sqrt{\sigma_G}(\log T )^{1/4}} \sqrt{d \log \left(\frac{T}{d}\right)T}+\frac{2d \hat{\sigma}}{\kappa} \log \left( \frac{T}{d \delta}\right)\sqrt{T} \right. \\
&& \,\,+\left.2\sqrt{2T d \log \left(\frac{T}{d}\right)}{\left(\sqrt{\sigma_G } \left({2\log\left( \frac{1}{\delta}\right)+2\log C_3 \sigma_G \sqrt{2\log T}+2\log C_3}\right)^{1/4}+\sqrt{1+2\sigma_G^2\log C_3}\right)}\right] 
\end{eqnarray*}
\end{proof}

\begin{proof}[Proof of Proposition \ref{prop:Gt}.]
When there exists an upper bound $D_{\max}$ on the delay, Proposition \ref{prop:delay} can be improved as follows.

Then there exist positive, universal constants $C_1$ and $C_2$ such that $\lambda_{\min}(W_t)\geq B$ with probability at least $1-\delta$, as long as 
\begin{eqnarray*}
t \geq \left( \frac{C_1\sqrt{d}+C_2\sqrt{\log(\frac{1}{\delta})}}{\lambda_{\min}(\Sigma)}\right)^2 +\frac{2B}{\lambda_{\min}(\Sigma)}+  D_{\max}.
\end{eqnarray*}
Along with the fact that event $\EE_{\Delta}$ holds for all $t \geq \tau$ with probability at least $1-2\delta$, we have with probability $1-3\delta$,
\[
\eqref{regret_inequaltiy} \leq  \tau + L_{g} \left( 2{\sqrt{D_{\max}}} \sqrt{2T d \log \left(\frac{T}{d}\right)} + \frac{2d \hat{\sigma}}{\kappa} \log \left(\frac{T}{d\delta}\right)\sqrt{T}\right).
\]
That is, $O(R_T) = O(D_{\max}\sqrt{dT\log(T)}+d\sqrt{T}\log(T))$

When $\{D_t\}_{t=1}^T$ are  \textbf{iid} with mean $\mu_I$,

\begin{eqnarray*}
\E[G_t] &=& \E[\sum_{s=1}^{t-1} 1_{s+D_s \geq t}] = \sum_{s=1}^{t-1}\P(s+D_s \geq t) \leq \mu_I,\\
\V[G_t] &=&\V[\sum_{s=1}^{t-1} 1_{s+D_s \geq t}] \leq \sum_{s=1}^{t-1}\P(s+D_s \geq t) \leq \mu_I.
\end{eqnarray*}
Therefore, with probability $1-5\delta$,
\begin{eqnarray*}
\eqref{regret_inequaltiy} \leq && \tau+ L_{g} \left[ {4\sqrt{\mu_I}} \sqrt{T d \log \left(\frac{T}{d}\right)} + {2^{7/4}\sqrt{\sigma_G}(\log T )^{1/4}} \sqrt{d \log \left(\frac{T}{d}\right)T}+\frac{2d \hat{\sigma}}{\kappa} \log \left( \frac{T}{d \delta}\right)\sqrt{T} \right. \\
&& \,\,+\left.2\sqrt{2T d \log \left(\frac{T}{d}\right)}\left(\sqrt{\sigma_G } \left({2\log\left( \frac{1}{\delta}\right)+2\log C_3 \sigma_G \sqrt{2\log T}+2\log C_3}\right)^{1/4}\right.\right.\\
&&\,\,\left.\left.+\sqrt{1+2\sigma_G^2\log C_3}\right)\right].
\end{eqnarray*}
\end{proof}

\begin{proof}[Proof of Theorem \ref{alg:BayesRegret}.]

Define $f_{\theta}(x) = g( x^{\prime}\theta)$ and denote $$\Theta_t := \left\{\theta \in \R^d \,\,\,\left\vert\,\,\, \|\theta-\hat{\theta}_t\|_{W_t}\leq \frac{\hat{\sigma}}{\kappa} \sqrt{\frac{d}{2} \log \left(1+\frac{2(t-G_t)}{d}\right)+\log(\frac{1}{\delta})}\right\}\right.$$ as an ellipsoidal confidence set centered around the MLE estimator $\hat{\theta}_t$ at round $t$.

Since reward $y_{t,a_t}\in[0,1]$ for all $t\geq 1$, denote the confidence bound $U_t(a) :=  \min \{1,\max_{\rho \in \Theta_t} g \left( \rho^{\prime} x_{t,a} \right)\}$
and $L_t(a) :=  \max\{0,\min_{\rho \in \Theta_t} g \left(\rho^{\prime} x_{t,a} \right)\}$.

Recall that $a_t = \arg\max_{a \in [K]}U_t(a)$ and $a^*_t \in \arg\max f_{\theta^*}(x)$, therefore we have the following simple regret decomposition:
\begin{eqnarray}\label{UCB-decomposition}
f_{\theta^*}(a_t^*)-f_{\theta^*}(a_t) &=& f_{\theta^*}(a_t^*)-U_t(a_t)+U_t(a_t)-f_{\theta^*}(a_t) \nonumber\\
&\leq&[f_{\theta^*}(a_t^*)-U_t(a^*_t)]+[U_t(a_t)-f_{\theta^*}(a_t)].
\end{eqnarray}

Taking the expectation of \eqref{UCB-decomposition} {\color{black}with respect to the prior distribution on $\theta^*$ and feature distribution $\gamma$ on $\{x_t\}_{t=1}^T$} leads to the $T$-period Bayesian regret of a UCB algorithm,
\begin{eqnarray}\label{BayesRegretUCB}
R^B_T(\pi^{U}) \leq \E \sum_{t=1}^T [U_t (a_t)-f_{\theta^*}(a_t)]+\E \sum_{t=1}^T [f_{\theta^*}(a_t^*)-U_t(a^*_t)],
\end{eqnarray}
where $\pi^{U}$ is the policy derived from $U:=\{U_t\}_{t=1}^T$.

Recall that for any UCB sequence $\{U_t \vert t \in \N \}$,
\begin{eqnarray}\label{BayesRegret}
R^B_T(\pi^{\text{PS}})  = \E \sum_{t=1}^T [U_t (\bar{a}_t)-f_{\theta^*}(\bar{a}_t)]+\E \sum_{t=1}^T [f_{\theta^*}(a_t^*)-U_t(a^*_t)],
\end{eqnarray}
where $\{\bar{a}_t\}_{t=1}^T$ are the actions selected by posterior samplings (\cite{RV2014}).

{Since $f_{\theta^*}$ takes values in $[0, 1]$}, from \eqref{BayesRegretUCB} and \eqref{BayesRegret},
\begin{eqnarray}\label{BayesRegretUCB2}
R^B_T(\pi^{{U}})  \leq \E \sum_{t=1}^T [U_t (a_t)-f_{\theta^*}(a_t)]+  \sum_{t=1}^T \P \left(f_{\theta^*}(a_t^*)>U_t(a^*_t)\right),
\end{eqnarray}

and 
\begin{eqnarray}\label{BayesRegret2}
R^B_T(\pi^{\text{PS}})  \leq \E \sum_{t=1}^T [U_t (\bar{a}_t)-f_{\theta^*}(\bar{a}_t)]+ \sum_{t=1}^T \P \left(f_{\theta^*}(a_t^*)>U_t(a^*_t)\right).
\end{eqnarray}
\eqref{BayesRegret2} implies
\begin{eqnarray}
R^B_T(\pi_{\tau}^{\text{PS}}) \leq \tau +  \E \sum_{t=\tau+1}^T [U_t (\bar{a}_t)-f_{\theta^*}(\bar{a}_t)]+ \sum_{t=\tau+1}^T \P \left(f_{\theta^*}(a_t^*)>U_t(a^*_t)\right)
\end{eqnarray}

If the sequence of
confidence parameters $\beta_1,\beta_2,\cdots,\beta_T$ is selected so that $\P(\theta^* \not\in \Theta_t)\leq \frac{1}{T}$ then the second term of the regret decomposition is less than $1$.

Our next task is to bound $ \E \sum_{t=\tau+1}^T [U_t (\bar{a}_t)-f_{\theta^*}(\bar{a}_t)]$.
Denote $\theta_t^{U} \in \arg\max_{\rho \in \Theta_t} g(\langle x_{t,\bar{a}_t}, \rho\rangle)$ and $\theta_t^{L} \in \arg\min_{\rho \in \Theta_t} g(\langle x_{t,\bar{a}_t}, \rho\rangle)$ where $\bar{a}_t$ is the action from Algorithm \ref{alg:PS-DGLM} at round $t$ for $t \geq \tau+1$. Therefore,

\begin{eqnarray*}
\langle \theta_t^{U},x_{t,\bar{a}_t} \rangle - \langle \theta_t^{L},x_{t,\bar{a}_t} \rangle 
&\leq& \langle \theta_t^{U}-\hat{\theta}_{t},x_{t,\bar{a}_t} \rangle - \langle \hat{\theta}_{t}-\theta_t^{L},x_{t,\bar{a}_t} \rangle\\
&\leq& \|\theta_t^{U}-\hat{\theta}_{t}\|_{V_{t}} \|x_{t,\bar{a}_t}\|_{V^{-1}_{t}} +  \|\theta_t^{L}-\hat{\theta}_{t}\|_{V_{t}} \|x_{t,\bar{a}_t}\|_{V^{-1}_{t}}\\
&\leq& 2 {\beta_{t}} \|x_{t,\bar{a}_t}\|_{V^{-1}_{t}}\\
&\leq& 2 {\beta_{t}} \min \{\|x_{t,\bar{a}_t}\|_{V^{-1}_{t}},1\}.
\end{eqnarray*}
where $\hat{\theta}_t$ is the MLE estimator from Algorithm \ref{DUCB-GLCB} at round $t$.
Note that  with probability $1-\delta=1-\frac{1}{T}$, $\theta^* \in\Theta_t$ holds for all $t\geq \tau$. Therefore we have with probability $1-\frac{1}{T}$,

\begin{eqnarray}\label{inequality}
  \sum_{t=\tau+1}^T [U_t (\bar{a}_t)-f_{\theta^*}(\bar{a}_t)] &\leq & \sum_{t=1}^T L_g \left(\langle \theta_t^{U},x_{t,\bar{a}_t} \rangle - \langle \theta_t^{L},x_{t,\bar{a}_t} \rangle\right) \nonumber\\
 &\leq & 2 L_{g} \left( {\max_{\tau+1 \leq t \leq T}{\beta_{t}}} \right) {\color{black}}\sum_{t=\tau+1}^T \min \left\{1,\|x_{t,\bar{a}_t}\|_{V^{-1}_{t}}\right\}\nonumber\\
  &\leq & 2 L_{g} {\color{black}}\left( {\max_{\tau+1 \leq t \leq T}{\beta_{t}}}\right) \sqrt{2(T-\tau)d \log \left(\frac{T}{d} \right)}. 
\end{eqnarray}
The last inequality holds thanks to Lemma 2 in (\cite{LLZ2017}).

Recall that $\beta_t = \frac{\hat{\sigma}}{\kappa} \sqrt{\frac{d}{2} \log \left(1+\frac{2(t-G_t)}{d}\right)+\log(\frac{1}{\delta})}+{\sqrt{G_t}}$, we have  
$$  {\max_{\tau+1 \leq t \leq T}{\beta_{t}}}  \leq \frac{\hat{\sigma}}{\kappa} \sqrt{\frac{d}{2} \log \left(1+\frac{2T}{d}\right)+\log(\frac{1}{\delta})}+{\sqrt{G_T^*}}.$$
 
 Take $\delta=\frac{1}{T}$ and use the fact $G_T \leq T$, we have
 $\E[G_T^*]\leq  2( \mu+M)+ \sigma_G \sqrt{2\log T}+ 2\sigma_G^2\log C_3 +\sqrt{2\sigma_G^2}\sqrt{\log\left( T\right)+\log C_3 \sigma_G \sqrt{2\log T}+\log C_3}+2 $. {Along with the fact that $\E[\sqrt{G_T^*}] \leq \sqrt{\E[\sqrt{G_T^*}]}$, we have}

\begin{eqnarray*}
\E \left(  {\max_{\tau+1 \leq t \leq T}{\beta_{t}}} \right) &\leq& \frac{\hat{\sigma}}{\kappa} \sqrt{\frac{d}{2} \log \left(1+\frac{2T}{d}\right)+\log T}+{\sqrt{2( \mu+M)}+ \sqrt{\sigma_G} ({2\log T})^{1/4}}\\
&&{+ \sqrt{2\sigma_G^2\log C_3 +2}+\left({2\sigma_G^2}\left(\log\left( T\right)+\log C_3 \sigma_G \sqrt{2\log T}+\log C_3\right)\right)^{1/4}}.
\end{eqnarray*}

Therefore, $\theta^* \in \Theta_t$ holds with probability $1-\frac{1}{T}$, and 
\begin{eqnarray*}
 \E\sum_{t=\tau+1}^T [U_t (\bar{a}_t)&-&f_{\theta^*}(\bar{a}_t) \vert \theta^* \in \Theta_t] \\&\leq& 2L_{g} \left( \frac{\hat{\sigma}}{\kappa} \sqrt{\frac{d}{2} \log \left(1+\frac{2T}{d}\right)+\log T}+{\sqrt{2( \mu+M)}+ \sqrt{\sigma_G} ({2\log T})^{1/4}}\right.\\
 &+&\left. {\sqrt{2\sigma_G^2\log C_3 +2}+\left({2\sigma_G^2}\left(\log\left( T\right)+\log C_3 \sigma_G \sqrt{2\log T}+\log C_3\right)\right)^{1/4}} \right)\sqrt{2Td \log \left(\frac{T}{d} \right)}.
\end{eqnarray*}

Combining with the fact that $ \sum_{t=\tau+1}^T [U_t (\bar{a}_t)-f_{\theta^*}(\bar{a}_t)]\leq T-\tau$ holds almost surely, we have 
$$R^B_T(\pi_{\tau}^{\text{PS}}) = {O\left( d\log T \sqrt{T}+\sqrt{\sigma_G}\sqrt{Td}(\log(T))^{3/4}+(\sqrt{\mu+M}+\sigma_G)\sqrt{dT\log \left(T\right)}\right).}$$
\end{proof}

\begin{proof}[Proof of Corollary \ref{ts-cor}.]
When $\{D_s\}_{s=1}^T$ are \textbf{iid} with mean $\mu_I$, $\E[G_T^*]\leq  2\mu_I+ \sigma_G \sqrt{2\log T}+ 2\sigma_G^2\log C_3 +\sqrt{2\sigma_G^2}\sqrt{\log\left( T\right)+\log C_3 \sigma_G \sqrt{2\log T}+\log C_3}+2 $. Similar to the proof of Theorem \ref{alg:BayesRegret}, take 
$\delta=\frac{1}{T}$, 
\begin{eqnarray*}
\E\sum_{t=\tau+1}^T [U_t (\bar{a}_t)&-&f_{\theta^*}(\bar{a}_t) \vert \theta^* \in \Theta_t]\\
&\leq& 2L_{g} \left( \frac{\hat{\sigma}}{\kappa} \sqrt{\frac{d}{2} \log \left(1+\frac{2T}{d}\right)+\log T}+{\sqrt{2\mu_I}+ \sqrt{\sigma_G} ({2\log T})^{1/4}}\right.\\
 &+&\left. {\sqrt{2\sigma_G^2\log C_3 +2}+\left({2\sigma_G^2}\left(\log\left( T\right)+\log C_3 \sigma_G \sqrt{2\log T}+\log C_3\right)\right)^{1/4}} \right)\sqrt{2Td \log \left(\frac{T}{d} \right)}.
\end{eqnarray*}
Therefore, 
$R^B_T(\pi_{\tau}^{\text{PS}}) ={O\left( d\log T \sqrt{T}+\sqrt{\sigma_G}\sqrt{Td}(\log(T))^{3/4}+(\sqrt{\mu_I}+\sigma_G)\sqrt{dT\log \left(T\right)}\right).}$

When $\{D_s\}_{s=1}^T$ are bounded by $D_{\max}$, take 
$\delta=\frac{1}{T}$, we have
\begin{eqnarray*}
\E\sum_{t=\tau+1}^T [U_t (\bar{a}_t)-f_{\theta^*}(\bar{a}_t) \vert \theta^* \in \Theta_t]&\leq& 2L_{g} \left( \frac{\hat{\sigma}}{\kappa} \sqrt{\frac{d}{2} \log \left(1+\frac{2T}{d}\right)+\log(T)}+{\sqrt{D_{\max}}} \right)\sqrt{2Td \log \left(\frac{T}{d} \right)}.\\
\end{eqnarray*}
Therefore, 
$$R^B_T(\pi_{\tau}^{\text{PS}}) = O\left( d\log T\sqrt{T}+{\sqrt{D_{\max}}}\sqrt{dT\log T}\right).$$
\end{proof}

\begin{proof}[Proof of Proposition \ref{prop:random}.]

Recall $c=\frac{1}{\sum_{i=1}^{\infty}\frac{1}{i^{1+q}}}$. Define $s_i = \frac{c}{i^{1+q}}$, therefore $\sum_{i=1}^{\infty}s_i=1$. 

Define $\tilde{G}_{\infty} = \sum_{t=1}^{\infty} \I\{\tilde{D}_t-\mu_R \geq t\}$, where $\{\tilde{D}_t\}_{t=1}^{\infty}$ is a process satisfying Assumption \ref{random_delay}.
By utilizing the following equation and the union bound,
\[
\P(G_{\infty}>x) = \P\left(\sum_{t=1}^{\infty}\I\{\tilde{D}_t -\mu_R\geq t\}>x\sum_{t=1}^{\infty}s_t\right),
\]
we have
\begin{eqnarray*}
\P(G_{\infty}>x) &\leq& \sum_{t=1}^{\infty} \P\left(\I\{\tilde{D}_t-\mu_R \geq t\}>x s_t\right)\\
&=& \sum_{t=1}^{\infty} \P(\tilde{D}_t-\mu_R \geq t)\I(xs_t<1)\\
&\leq& \sum_{t=\sqrt[1+q]{cx}}^{\infty}\P(\tilde{D}_t-\mu_R \geq t) \leq \sum_{t=\sqrt[1+q]{cx}}^{\infty} \exp\left(-\frac{t^{2(1+q)}}{2\sigma_D^2}\right)\\
&\leq& C_4 \exp \left(-\frac{x^2}{2\frac{{\sigma}^2_R}{c^2}} \right),
\end{eqnarray*}
for some $C_4\leq 2\sigma_R^2+1$.

Define $G_{\infty} = \sum_{t=1}^{\infty} \I\{D_t\geq t\}$, then $G_{\infty}\leq \mu_R + \tilde{G}_{\infty}$ with $\tilde{D}_t={D}_{t+\mu_R}$ for $t\geq 1$. Therefore, 
\[
\P(G_{\infty}-\mu_R>x)\leq C_4 \exp \left(-\frac{x^2}{2\frac{{\sigma}^2_R}{c^2}} \right),
\]
which implies the sub-Gaussian property of $G_{\infty}$.

Similarly, we can show that $G_t$ is sub-Gaussian with parameters $\left(C_4,\sqrt{\frac{\sigma_D}{c}}\right)$ for all $t \geq 1$.
\end{proof}

\begin{proof}[Proof of Proposition \ref{prop:delay_second_choice}]
Here we only show the sub-Gaussian property for $G_t$ by using the Hoeffding bound. The rest of the proof follows Proposition \ref{delay}.

Again, let $\tilde{D}_{k_i}$ be a random variable such that $\tilde{D}_{k_i}\geq -(\mu+M)$ almost surely, $\E[\tilde{D}_{k_i}]\leq 0$ and $\P(\tilde{D}_{k_i}\geq x) \leq \exp{\left(-\frac{x^{1+q}}{2\sigma^2}\right)}$ for $x \geq 0$. One can view $\tilde{D}_{k_i}$ as a shifted delay.

Define $\tilde{I}_i = \I \left(\tilde{D}_{k_i} \geq i\right)-p_i$ with $p_i = \P(\tilde{D}_{k_i} \geq i)$.
Then $\P \left(\tilde{I}_i=1-p_i\right)=p_i$ and $\P(\tilde{I}_i=p_i)=1-p_i$. Denote $\sigma_i = \sqrt{\frac{1-2p_i}{2 \log \left( \frac{1-p_i}{p_i} \right)}}$, it is easy to verify that
\[
\E \exp \left({\lambda\tilde{I}_i }\right)=p_i \exp(\lambda (1-p_i)) + (1-p_i)\exp(-p_i\lambda) \leq \exp\left(\frac{\sigma^2_i \lambda^2}{2}\right).
\]
Therefore $\tilde{I}_i$ is sub-Gaussian with parameter $\sigma_i$. (Also see Theorem \ref{thm8}.)

We first show that when $i \geq \max\left\{\sqrt[1+q]{2\log (2) \sigma^2},\sqrt[q]{\frac{2\sigma^2}{1+q}}+1\right\}:=I$, we have
\begin{eqnarray}
&&p_i \leq \frac{1}{2}, \label{half}\\
\text{and }&&\exp\left(\frac{i^{1+q}}{2\sigma^2}\right)-\exp\left(\frac{(i-1)^{1+q}}{2\sigma^2}\right) \geq 1.\label{increment}
\end{eqnarray}

\begin{itemize}
\item When $i\geq\sqrt[1+q]{2\log (2) \sigma^2}$, $$p_i \leq e^{-\frac{i^{1+q}}{2 \sigma^2}}\leq \frac{1}{2}.$$ 

The first inequality holds by Assumption 2 and second inequality holds by simple calculation.

\item Define $h(x)=\exp\left(\frac{x^{1+q}}{2\sigma^2}\right)$ with $q >0$, which is differentiable. By the Mean Value Theorem, $h(x)-h(y)=\exp \left(\frac{z^{1+q}}{2\sigma^2}\right)\frac{(1+q)z^q}{2\sigma^2}(x-y)$ for some $z\in(x,y)$. Take $x=i-1$ and $y=i$, for some $z\in[i-1,i]$, we have
\begin{eqnarray}
\exp \left(\frac{i^{1+q}}{2\sigma^2}\right)-\exp \left(\frac{(i-1)^{1+q}}{2\sigma^2}\right) &=& \exp \left(\frac{z^{1+q}}{2\sigma^2}\right)\frac{(1+q)z^q}{2\sigma^2}\nonumber\\
&\geq& \frac{(1+q)z^q}{2\sigma^2}\geq \frac{(1+q)(i-1)^q}{2\sigma^2}
\geq 1.\label{inter2}
\end{eqnarray}
The last inequality in \eqref{inter2} holds since $i \geq \sqrt[q]{\frac{\sigma^2}{1+q}}+1$.
\end{itemize}
Given \eqref{half}-\eqref{increment}, when $i \geq I$ and $q \geq 0$,
\begin{eqnarray}
\sigma^2_i =  \frac{1-2p_i}{2 \log \left( \frac{1-p_i}{p_i} \right)}&\leq& \frac{1}{2\log\left( \frac{1-p_i}{p_i} \right) }\label{iter3}\\
&\leq&       \frac{\sigma^2}{ (i-1)^{1+q}} .   \label{iter4}
\end{eqnarray}
\eqref{iter3} holds since \eqref{half} and \eqref{iter4} holds since \eqref{increment}. Therefore

\begin{eqnarray*}
\sum _{i=I}^{\infty}\sigma^2_i &=&\sum _{i=I}^{\infty} \frac{1-2p_i}{2\log \left( \frac{1-p_i}{p_i} \right)}
\leq\sum _{i=I}^{\infty} \frac{1}{ 2\log \left( \frac{1-p_i}{p_i} \right)}\leq\sum_{i=I-1}^{\infty} \frac{\sigma^2}{ i^{1+q}} \\
&\leq& \sigma^2 \left(1+\sum_{i=2}^{\infty} \frac{1}{i^{1+q}}\right) \leq  \sigma^2 \left(1+\int_1^{\infty}\frac{1}{x^{(1+q)}}dx \right)
= \frac{\sigma^2 (1+q)}{q}.
\end{eqnarray*}

It is easy to check that $\sigma_i^2 =\frac{1-2p_i}{2 \log \left( \frac{1-p_i}{p_i} \right)} \leq \frac{1}{4}$ for all $p_i \in [0,1]$. Therefore, $\sum_{i=1}^{\infty}\sigma^2_i \leq \frac{1}{4}I +\frac{\sigma^2 (1+q)}{q}$.

Define $\tilde{G} = \sum_{i=1}^{\infty}\tilde{I}_i$.
 combining above with Theorem \ref{thm9}, $\tilde{G}$ is sub-Gaussian with parameter $\sigma_G = \sqrt{\frac{I}{4}+\frac{\sigma^2 (1+q)}{q}}$.
Similarly, we can show that $\tilde{G}_t = \sum_{i=1}^t \tilde{I}_i$ is sub-Gaussian with parameter $\sigma_G = \sqrt{\frac{I}{4}+\frac{\sigma^2 (1+q)}{q}}$ as well.
\end{proof}

%% file: delay_arxiv.bbl
\begin{thebibliography}{10}

\bibitem{hsu2014taming}
Alekh Agarwal, Daniel Hsu, Satyen Kale, John Langford, Lihong Li, and Robert
  Schapire.
\newblock Taming the monster: A fast and simple algorithm for contextual
  bandits.
\newblock In {\em International Conference on Machine Learning}, pages
  1638--1646, 2014.

\bibitem{agrawal2017thompson}
Shipra Agrawal, Vashist Avadhanula, Vineet Goyal, and Assaf Zeevi.
\newblock Thompson sampling for the mnl-bandit.
\newblock {\em arXiv preprint arXiv:1706.00977}, 2017.

\bibitem{agrawal2016efficient}
Shipra Agrawal, Nikhil~R Devanur, and Lihong Li.
\newblock An efficient algorithm for contextual bandits with knapsacks, and an
  extension to concave objectives.
\newblock In {\em Conference on Learning Theory}, pages 4--18, 2016.

\bibitem{AG2013a}
Shipra Agrawal and Navin Goyal.
\newblock Further optimal regret bounds for thompson sampling.
\newblock In {\em Artificial intelligence and statistics}, pages 99--107, 2013.

\bibitem{AG2013b}
Shipra Agrawal and Navin Goyal.
\newblock Thompson sampling for contextual bandits with linear payoffs.
\newblock In {\em International Conference on Machine Learning}, pages
  127--135, 2013.

\bibitem{athey2017efficient}
Susan Athey and Stefan Wager.
\newblock Efficient policy learning.
\newblock {\em arXiv preprint arXiv:1702.02896}, 2017.

\bibitem{bastani2015online}
Hamsa Bastani and Mohsen Bayati.
\newblock Online decision-making with high-dimensional covariates.
\newblock 2015.

\bibitem{besbes2009dynamic}
Omar Besbes and Assaf Zeevi.
\newblock Dynamic pricing without knowing the demand function: Risk bounds and
  near-optimal algorithms.
\newblock {\em Operations Research}, 57(6):1407--1420, 2009.

\bibitem{BCN2012}
S{\'e}bastien Bubeck, Nicolo Cesa-Bianchi, et~al.
\newblock Regret analysis of stochastic and nonstochastic multi-armed bandit
  problems.
\newblock {\em Foundations and Trends in Machine Learning}, 5(1):1--122, 2012.

\bibitem{CGM2019}
Nicolo Cesa-Bianchi, Claudio Gentile, and Yishay Mansour.
\newblock Delay and cooperation in nonstochastic bandits.
\newblock {\em The Journal of Machine Learning Research}, 20(1):613--650, 2019.

\bibitem{chapelle2014}
Olivier Chapelle.
\newblock Modeling delayed feedback in display advertising.
\newblock In {\em Proceedings of the 20th ACM SIGKDD international conference
  on Knowledge discovery and data mining}, pages 1097--1105. ACM, 2014.

\bibitem{CL2011}
Olivier Chapelle and Lihong Li.
\newblock An empirical evaluation of thompson sampling.
\newblock In {\em Advances in neural information processing systems}, pages
  2249--2257, 2011.

\bibitem{CHY1999}
Kani Chen, Inchi Hu, Zhiliang Ying, et~al.
\newblock Strong consistency of maximum quasi-likelihood estimators in
  generalized linear models with fixed and adaptive designs.
\newblock {\em The Annals of Statistics}, 27(4):1155--1163, 1999.

\bibitem{CC2011}
Shein-Chung Chow and Mark Chang.
\newblock {\em Adaptive design methods in clinical trials}.
\newblock Chapman and Hall/CRC, 2011.

\bibitem{chu2011contextual}
Wei Chu, Lihong Li, Lev Reyzin, and Robert Schapire.
\newblock Contextual bandits with linear payoff functions.
\newblock In {\em Proceedings of the Fourteenth International Conference on
  Artificial Intelligence and Statistics}, pages 208--214, 2011.

\bibitem{DKVB2014}
Thomas Desautels, Andreas Krause, and Joel~W Burdick.
\newblock Parallelizing exploration-exploitation tradeoffs in gaussian process
  bandit optimization.
\newblock {\em The Journal of Machine Learning Research}, 15(1):3873--3923,
  2014.

\bibitem{DY1979}
Persi Diaconis and Donald Ylvisaker.
\newblock Conjugate priors for exponential families.
\newblock {\em The Annals of statistics}, pages 269--281, 1979.

\bibitem{DHKKLRZ2011}
Miroslav Dudik, Daniel Hsu, Satyen Kale, Nikos Karampatziakis, John Langford,
  Lev Reyzin, and Tong Zhang.
\newblock Efficient optimal learning for contextual bandits.
\newblock {\em arXiv preprint arXiv:1106.2369}, 2011.

\bibitem{dudik2011doubly}
Miroslav Dud{\'\i}k, John Langford, and Lihong Li.
\newblock Doubly robust policy evaluation and learning.
\newblock {\em arXiv preprint arXiv:1103.4601}, 2011.

\bibitem{FCGS2010}
Sarah Filippi, Olivier Cappe, Aur{\'e}lien Garivier, and Csaba Szepesv{\'a}ri.
\newblock Parametric bandits: The generalized linear case.
\newblock In {\em Advances in Neural Information Processing Systems}, pages
  586--594, 2010.

\bibitem{GST2016}
Scott Garrabrant, Nate Soares, and Jessica Taylor.
\newblock Asymptotic convergence in online learning with unbounded delays.
\newblock {\em arXiv preprint arXiv:1604.05280}, 2016.

\bibitem{goldenshluger2011note}
Alexander Goldenshluger and Assaf Zeevi.
\newblock A note on performance limitations in bandit problems with side
  information.
\newblock {\em IEEE Transactions on Information Theory}, 57(3):1707--1713,
  2011.

\bibitem{JSF2018}
Bai Jiang, Qiang Sun, and Jianqing Fan.
\newblock Bernstein's inequality for general markov chains.
\newblock {\em arXiv preprint arXiv:1805.10721}, 2018.

\bibitem{deep-learning-logged-bandit-feedback}
Thorsten Joachims, Adith Swaminathan, and Maarten~de Rijke.
\newblock Deep learning with logged bandit feedback.
\newblock In {\em International Conference on Learning Representations}, May
  2018.

\bibitem{JGS2013}
Pooria Joulani, Andras Gyorgy, and Csaba Szepesv{\'a}ri.
\newblock Online learning under delayed feedback.
\newblock In {\em International Conference on Machine Learning}, pages
  1453--1461, 2013.

\bibitem{JBNW2017}
Kwang-Sung Jun, Aniruddha Bhargava, Robert Nowak, and Rebecca Willett.
\newblock Scalable generalized linear bandits: Online computation and hashing.
\newblock In {\em Advances in Neural Information Processing Systems}, pages
  99--109, 2017.

\bibitem{KCW2001}
PK~Kannan, A-M Chang, and Andrew~B Whinston.
\newblock Wireless commerce: marketing issues and possibilities.
\newblock In {\em Proceedings of the 34th Annual Hawaii International
  Conference on System Sciences}, pages 6--pp. IEEE, 2001.

\bibitem{kitagawa2018should}
Toru Kitagawa and Aleksey Tetenov.
\newblock Who should be treated? empirical welfare maximization methods for
  treatment choice.
\newblock {\em Econometrica}, 86(2):591--616, 2018.

\bibitem{LCLS2010}
Lihong Li, Wei Chu, John Langford, and Robert~E Schapire.
\newblock A contextual-bandit approach to personalized news article
  recommendation.
\newblock In {\em Proceedings of the 19th international conference on World
  wide web}, pages 661--670. ACM, 2010.

\bibitem{LLZ2017}
Lihong Li, Yu~Lu, and Dengyong Zhou.
\newblock Provably optimal algorithms for generalized linear contextual
  bandits.
\newblock In {\em Proceedings of the 34th International Conference on Machine
  Learning-Volume 70}, pages 2071--2080. JMLR. org, 2017.

\bibitem{MLBP2015}
Travis Mandel, Yun-En Liu, Emma Brunskill, and Zoran Popovi{\'c}.
\newblock The queue method: Handling delay, heuristics, prior data, and
  evaluation in bandits.
\newblock In {\em Twenty-Ninth AAAI Conference on Artificial Intelligence},
  2015.

\bibitem{McCullagh2018}
Peter McCullagh.
\newblock {\em Generalized linear models}.
\newblock Routledge, 2018.

\bibitem{mesterharm2005}
Chris Mesterharm.
\newblock On-line learning with delayed label feedback.
\newblock In {\em International Conference on Algorithmic Learning Theory},
  pages 399--413. Springer, 2005.

\bibitem{NW1972}
John~Ashworth Nelder and Robert~WM Wedderburn.
\newblock Generalized linear models.
\newblock {\em Journal of the Royal Statistical Society: Series A (General)},
  135(3):370--384, 1972.

\bibitem{NAGS2010}
Gergely Neu, Andras Antos, Andr{\'a}s Gy{\"o}rgy, and Csaba Szepesv{\'a}ri.
\newblock Online markov decision processes under bandit feedback.
\newblock In {\em Advances in Neural Information Processing Systems}, pages
  1804--1812, 2010.

\bibitem{OS2014}
Eugene Ostrovsky and Leonid Sirota.
\newblock Exact value for subgaussian norm of centered indicator random
  variable.
\newblock {\em arXiv preprint arXiv:1405.6749}, 2014.

\bibitem{PASG2017}
Ciara Pike-Burke, Shipra Agrawal, Csaba Szepesvari, and Steffen Grunewalder.
\newblock Bandits with delayed, aggregated anonymous feedback.
\newblock {\em arXiv preprint arXiv:1709.06853}, 2017.

\bibitem{QK2015}
Kent Quanrud and Daniel Khashabi.
\newblock Online learning with adversarial delays.
\newblock In {\em Advances in neural information processing systems}, pages
  1270--1278, 2015.

\bibitem{rigollet2010nonparametric}
Philippe Rigollet and Assaf Zeevi.
\newblock Nonparametric bandits with covariates.
\newblock {\em arXiv preprint arXiv:1003.1630}, 2010.

\bibitem{RV2014}
Daniel Russo and Benjamin Van~Roy.
\newblock Learning to optimize via posterior sampling.
\newblock {\em Mathematics of Operations Research}, 39(4):1221--1243, 2014.

\bibitem{russo2016information}
Daniel Russo and Benjamin Van~Roy.
\newblock An information-theoretic analysis of thompson sampling.
\newblock {\em The Journal of Machine Learning Research}, 17(1):2442--2471,
  2016.

\bibitem{SBF2017}
Eric~M Schwartz, Eric~T Bradlow, and Peter~S Fader.
\newblock Customer acquisition via display advertising using multi-armed bandit
  experiments.
\newblock {\em Marketing Science}, 36(4):500--522, 2017.

\bibitem{swaminathan2015batch}
Adith Swaminathan and Thorsten Joachims.
\newblock Batch learning from logged bandit feedback through counterfactual
  risk minimization.
\newblock {\em Journal of Machine Learning Research}, 16:1731--1755, 2015.

\bibitem{off-policy-evaluation-slate-recommendation}
Adith Swaminathan, Akshay Krishnamurthy, Alekh Agarwal, Miro Dudík, John
  Langford, Damien Jose, and Imed Zitouni.
\newblock Off-policy evaluation for slate recommendation.
\newblock pages 3632--3642. Curran Associates, Inc., December 2017.

\bibitem{VERSHYNIN2010}
Roman Vershynin.
\newblock Introduction to the non-asymptotic analysis of random matrices.
\newblock {\em arXiv preprint arXiv:1011.3027}, 2010.

\bibitem{wainwright2019}
Martin~J Wainwright.
\newblock {\em High-dimensional statistics: A non-asymptotic viewpoint},
  volume~48.
\newblock Cambridge University Press, 2019.

\bibitem{WO2002}
Marcelo~J Weinberger and Erik Ordentlich.
\newblock On delayed prediction of individual sequences.
\newblock {\em IEEE Transactions on Information Theory}, 48(7):1959--1976,
  2002.

\bibitem{zhou2018offline}
Zhengyuan Zhou, Susan Athey, and Stefan Wager.
\newblock Offline multi-action policy learning: Generalization and
  optimization.
\newblock {\em arXiv preprint arXiv:1810.04778}, 2018.

\bibitem{zhou2019}
Zhengyuan Zhou, Renyuan Xu, and Jose Blanchet.
\newblock Learning in generalized linear contextual bandits with stochastic
  delays.
\newblock In {\em Advances in Neural Information Processing Systems}, 2019.

\end{thebibliography}
